\documentclass{article}
\usepackage{xr}
\usepackage{cite}
\usepackage{amsmath}
\usepackage[colorlinks=true,linkcolor=black,citecolor=black,urlcolor=black]{hyperref}
\usepackage{cases}
\usepackage[utf8]{inputenc}
\usepackage[english]{babel}
\usepackage{footnote}
\usepackage{bm}
\usepackage{multirow}
\usepackage{enumitem}

\usepackage[nonatbib,final]{neurips_2020}
\usepackage[T1]{fontenc}    
\usepackage{url}            
\usepackage{booktabs}       
\usepackage{amsfonts}       
\usepackage{nicefrac}       
\usepackage{microtype}      
\usepackage{rotating}

\usepackage[mathscr]{eucal}
\usepackage{epsfig,epsf,psfrag}
\usepackage{amssymb,amsmath,amsfonts,latexsym}
\usepackage{amsmath,graphicx,bm,xcolor,url}
\usepackage[caption=false]{subfig} 
\usepackage{fixltx2e}
\usepackage{array}
\usepackage{verbatim}
\usepackage{bm}
\usepackage{algpseudocode}
\usepackage{algorithm}
\usepackage{verbatim}
\usepackage{textcomp}
\usepackage{mathrsfs}
\usepackage{epstopdf}
\usepackage{relsize}
\usepackage{cleveref} 
\usepackage{subfig}
 \usepackage{amsthm}

\catcode`~=11 \def\UrlSpecials{\do\~{\kern -.15em\lower .7ex\hbox{~}\kern .04em}} \catcode`~=13 

\allowdisplaybreaks[3]

\newcommand{\nn}{\nonumber}

\newcommand{\calD}{\mathcal{D}}

\newcommand{\calK}{\mathcal{K}}

\newcommand{\calM}{\mathcal{M}}
\newcommand{\calN}{\mathcal{N}}

\newcommand{\calP}{\mathcal{P}}

\newcommand{\calS}{\mathcal{S}}

\newcommand{\calX}{\mathcal{X}}

\newcommand{\ba}{\mathbf{a}}
\newcommand{\bA}{\mathbf{A}}
\newcommand{\bb}{\mathbf{b}}
\newcommand{\bB}{\mathbf{B}}

\newcommand{\bg}{\mathbf{g}}

\newcommand{\bI}{\mathbf{I}}

\newcommand{\bs}{\mathbf{s}}

\newcommand{\bt}{\mathbf{t}}

\newcommand{\bu}{\mathbf{u}}

\newcommand{\bv}{\mathbf{v}}
\newcommand{\bV}{\mathbf{V}}

\newcommand{\bW}{\mathbf{W}}
\newcommand{\bx}{\mathbf{x}}
\newcommand{\bX}{\mathbf{X}}
\newcommand{\by}{\mathbf{y}}

\newcommand{\bz}{\mathbf{z}}

\newcommand{\rmd}{\mathrm{d}}

\newcommand{\rmH}{\mathrm{H}}

\newcommand{\bbE}{\mathbb{E}}

\newcommand{\bbP}{\mathbb{P}}

\newcommand{\bbR}{\mathbb{R}}

\DeclareMathAlphabet{\mathbsf}{OT1}{cmss}{bx}{n}
\DeclareMathAlphabet{\mathssf}{OT1}{cmss}{m}{sl}

\DeclareSymbolFont{bsfletters}{OT1}{cmss}{bx}{n}  
\DeclareSymbolFont{ssfletters}{OT1}{cmss}{m}{n}
\DeclareMathSymbol{\bsfGamma}{0}{bsfletters}{'000}
\DeclareMathSymbol{\ssfGamma}{0}{ssfletters}{'000}
\DeclareMathSymbol{\bsfDelta}{0}{bsfletters}{'001}
\DeclareMathSymbol{\ssfDelta}{0}{ssfletters}{'001}
\DeclareMathSymbol{\bsfTheta}{0}{bsfletters}{'002}
\DeclareMathSymbol{\ssfTheta}{0}{ssfletters}{'002}
\DeclareMathSymbol{\bsfLambda}{0}{bsfletters}{'003}
\DeclareMathSymbol{\ssfLambda}{0}{ssfletters}{'003}
\DeclareMathSymbol{\bsfXi}{0}{bsfletters}{'004}
\DeclareMathSymbol{\ssfXi}{0}{ssfletters}{'004}
\DeclareMathSymbol{\bsfPi}{0}{bsfletters}{'005}
\DeclareMathSymbol{\ssfPi}{0}{ssfletters}{'005}
\DeclareMathSymbol{\bsfSigma}{0}{bsfletters}{'006}
\DeclareMathSymbol{\ssfSigma}{0}{ssfletters}{'006}
\DeclareMathSymbol{\bsfUpsilon}{0}{bsfletters}{'007}
\DeclareMathSymbol{\ssfUpsilon}{0}{ssfletters}{'007}
\DeclareMathSymbol{\bsfPhi}{0}{bsfletters}{'010}
\DeclareMathSymbol{\ssfPhi}{0}{ssfletters}{'010}
\DeclareMathSymbol{\bsfPsi}{0}{bsfletters}{'011}
\DeclareMathSymbol{\ssfPsi}{0}{ssfletters}{'011}
\DeclareMathSymbol{\bsfOmega}{0}{bsfletters}{'012}
\DeclareMathSymbol{\ssfOmega}{0}{ssfletters}{'012}

\newcommand{\balpha}{\bm{\alpha}}

\newcommand{\btheta}{\bm{\theta}}

\newcommand{\bSigma	}{\bm{\Sigma}}

\newcommand{\floor}[1]{\lfloor{#1}\rfloor}

\newcommand{\bzero}{\mathbf{0}}

\theoremstyle{plain}
\newtheorem{theorem}{Theorem} 
\newtheorem{assumption}{Assumption} 
\newtheorem{lemma}{Lemma}

\newtheorem{corollary}{Corollary}
\newtheorem{definition}{Definition} 

\newtheorem{remark}{Remark}

\newcommand{\qednew}{\nobreak \ifvmode \relax \else
      \ifdim\lastskip<1.5em \hskip-\lastskip
      \hskip1.5em plus0em minus0.5em \fi \nobreak
      \vrule height0.75em width0.5em depth0.25em\fi}

\newcommand\sign{\mathrm{sign}}

\newcommand{\funcP}{P_{\rm LEP}}
\newcommand{\funcM}{M_{\rm LEP}}

\begin{document}
\title{The Generalized Lasso with Nonlinear \\ Observations and Generative Priors}

\author{%
Zhaoqiang Liu \\
    National University of Singapore\\
  \texttt{dcslizha@nus.edu.sg} \\
  \And
  Jonathan Scarlett \\
  National University of Singapore\\
  \texttt{scarlett@comp.nus.edu.sg} \\\
}

\maketitle

\begin{abstract}
    In this paper, we study the problem of signal estimation from noisy non-linear measurements when the unknown $n$-dimensional signal is in the range of an $L$-Lipschitz continuous generative model with bounded $k$-dimensional inputs. We make the assumption of sub-Gaussian measurements, which is satisfied by a wide range of measurement models, such as linear, logistic, 1-bit, and other quantized models.  In addition, we consider the impact of adversarial corruptions on these measurements.  Our analysis is based on a generalized Lasso approach (Plan and Vershynin, 2016). We first provide a non-uniform recovery guarantee, which states that under i.i.d.~Gaussian measurements, roughly $O\left(\frac{k}{\epsilon^2}\log L\right)$ samples suffice for recovery with an $\ell_2$-error of $\epsilon$, and that this scheme is robust to adversarial noise.  Then, we apply this result to neural network generative models, and discuss various extensions to other models and non-i.i.d.~measurements. Moreover, we show that our result can be extended to the uniform recovery guarantee under the assumption of a so-called local embedding property, which is satisfied by the 1-bit and censored Tobit models.  
\end{abstract}

\vspace*{-2ex}
\section{Introduction}\label{sec:intro}
\vspace*{-1ex}

In standard compressive sensing (CS)~\cite{Fou13,wainwright2019high}, one considers a linear observation model of the form
\begin{equation}
 y_i = \langle \ba_i, \bx^*\rangle + \epsilon_i, \quad i = 1,2,\ldots,m,
\end{equation}
where $\bx^* \in \bbR^n$ is an unknown $k$-sparse signal vector, $\ba_i \in \bbR^n$ is the $i$-th measurement vector, and $\epsilon_i \in \bbR$ is the noise term. The goal is to accurately recover $\bx^*$ given $\bA$ and $\by$, where $\bA \in \bbR^{m \times n}$ is the measurement matrix whose $i$-th row is $\ba_i^T$, and $\by \in \bbR^m$ is the vector of observations. To obtain an estimate of $\bx^*$, a natural idea is to minimize the $\ell_2$ loss subject to a structural constraint:
\begin{equation}\label{eq:gen_lasso_origin}
 \mathrm{minimize} \text{ } \|\bA\bx - \by\|_2  \quad \text{subject to} \quad \bx \in \calK,
\end{equation}
where $\calK$ captures the structure of $\bx^*$; this may be set to be the set of all $k$-sparse vectors in $\bbR^n$, or for computational reasons, may instead be the scaled $\ell_1$-ball, giving rise to the constrained Lasso \cite{tibshirani1996regression}. We refer to \eqref{eq:gen_lasso_origin} as the $\calK$-Lasso (with observations $\by$ and measurement matrix $\bA$).

Despite the far-reaching utility of standard CS, in many real-world applications the assumption of a linear model is too restrictive. To address this problem, the semi-parametric {\em single index model (SIM)} is considered in various papers~\cite{plan2016generalized,plan2017high,genzel2016high}:
\begin{equation}
 y_i = f(\langle \ba_i, \bx^*\rangle), \quad i = 1,2,\ldots,m, \label{eq:yi}
\end{equation}
where $f \,:\, \bbR \rightarrow \bbR$ is an unknown (possibly random) function that is independent of $\ba_i$. In general, $f$ plays the role of a nonlinearity, and we aim to estimate the signal $\bx^*$ despite this unknown nonlinearity. Note that the norm of $\bx^*$ is sacrificed in SIM, since it may be absorbed into the unknown function $f$. Hence, for simplicity of presentation, we assume that $\bx^*$ is a unit vector in $\bbR^n$. In this paper, similar to that in~\cite{plan2017high}, we make the assumption that the random variables $y_i$ are sub-Gaussian; see Section \ref{sec:example} for several examples.
In addition, to further strengthen the robustness guarantees, we also allow for adversarial noise. That is, we consider the case of corrupted observations $\tilde{\by}$ that can be produced from $\by$ in an arbitrary manner (possibly depending on $\bA$) subject to an $\ell_2$-norm constraint.

Motivated by the tremendous success of deep generative models in abundance of real applications~\cite{Fos19}, a new perspective of CS has recently emerged, in which the assumption that the underlying signal can be well-modeled by a (deep) generative model replaces the common sparsity assumption~\cite{bora2017compressed}. In addition to the theoretical developments, existing works have presented impressive numerical results for CS with generative models, with large reductions (e.g., a factor of 5 to 10) in the required number of measurements compared to sparsity-based methods \cite{bora2017compressed}.

\vspace*{-1ex}
\subsection{Related Work} \label{sec:related}
\vspace*{-1ex}

In this subsection, we provide a summary of some relevant existing works. These works can roughly be divided into (i) CS with generative models, and (ii) SIM without generative models. 

{\bf CS with generative models:} Bora {\em et al.}~\cite{bora2017compressed} show that for an $L$-Lipschitz continuous generative model with bounded $k$-dimensional inputs, roughly $O(k \log L)$ random Gaussian linear measurements suffice for an accurate recovery. 
The analysis in~\cite{bora2017compressed} is based on the $\calK$-Lasso~\eqref{eq:gen_lasso_origin}, as well as showing that a natural counterpart to the Restricted Eigenvalue Condition (REC), termed the Set-REC (S-REC), is satisfied by Gaussian measurement matrices. Extensive experimental results for the $\calK$-Lasso have been presented in~\cite{bora2017compressed} in the case of linear measurements. Follow-up works of~\cite{bora2017compressed} provide certain additional algorithmic guarantees~\cite{Sha18,peng2020solving,Dha18,Han18}, as well as information-theoretic lower bounds \cite{kamath2020power,liu2020information}.

1-bit CS with generative models has been studied in various recent works~\cite{qiu2020robust,liu2020sample}. In~\cite{qiu2020robust}, the authors study robust 1-bit recovery for $d$-layer, $w$-width ReLU neural network generative models, and the dithering technique~\cite{knudson2016one,xu2020quantized,jacques2017time,dirksen2018non} is used to enable the estimation of the norm. It is shown that roughly $O\big(\frac{kd}{\epsilon^2}\log w\big)$ sub-exponential measurements guarantee the uniform recovery\footnote{A {\em uniform recovery} guarantee is one in which some measurement matrix $\bA$ simultaneously ensures the recovery of all $\bx^*$ in the set of interest.  In contrast, {\em non-uniform recovery} only requires a randomly-drawn $\bA$ to succeed with high probability for fixed $\bx^*$. } of any signal in the range of the generative model up to an $\ell_2$-error of $\epsilon$. These results do not apply to general non-linear measurement models, and the authors only consider ReLU neural networks with no offsets, rather than general deep generative models. In addition, the algorithm analyzed is different to the $\calK$-Lasso. 

The authors of~\cite{liu2020sample} prove that the so-called Binary $\epsilon$-Stable Embedding property (B$\epsilon$SE) holds for 1-bit compressive sensing with $L$-Lipschitz continuous generative models with $O\big(\frac{k}{\epsilon^2}\log \frac{L}{\epsilon^2}\big)$ Gaussian measurements. However, these theoretical results are information-theoretic in nature, and computationally efficient algorithms are not considered.  Hence, when specialized to the 1-bit setting, our results complement those of \cite{liu2020sample} by considering the $\calK$-Lasso, which can be approximated efficiently via gradient descent \cite{bora2017compressed}.

The work~\cite{wei2019statistical} is perhaps closest to our work, 
and considers the estimation of a signal in the range of an $L$-Lipschitz continuous generative model from non-linear and heavy-tailed measurements. By considering estimators via score functions based on the first and second order Stein's identity, it is shown that roughly $O\big(\frac{k}{\epsilon^2}\log L\big)$ measurements suffice for achieving non-uniform recovery with $\ell_2$-error at most $\epsilon$. The authors make the assumption that the nonlinearity $f$ is {\em differentiable}, which fails to hold in several cases of interest (e.g., 1-bit and other quantized measurements).  In addition, the estimators based on the above-mentioned score functions differ significantly from the $\calK$-Lasso.

{\bf SIM without generative models: } The authors of~\cite{plan2017high} consider SIM and a measurement matrix with i.i.d.~standard Gaussian entries. The analysis is based on the estimates of both the global and local Gaussian mean width (GMW) of the set of structured signals $\calK$, which can be used to understand its effective dimension. The work~\cite{plan2016generalized} generalizes the results in~\cite{plan2017high} to allow $\ba_i \sim \calN(\mathbf{0},\bSigma)$ with an unknown covariance matrix $\bSigma$, derives tighter results when specialized to the linear model. In these papers, $\calK$ is assumed to be star-shaped\footnote{A set $\calK$ is star-shaped if $\lambda \calK \subseteq \calK$ whenever $0 \le \lambda \le 1$.} or convex, which may not be satisfied for the range of general Lipschitz continuous generative models. In addition, without further assumptions on the signal $\bx^*$, the structured set $\calK$, or the measurement model, the recovery error bound exhibits $m^{-\frac{1}{4}}$ scaling, which is weaker than the typical $m^{-\frac{1}{2}}$ scaling.  In each of these papers, only non-uniform recovery guarantees are provided. See the table in Appendix \ref{app:table} for a more detailed overview.

Further follow-up works of~\cite{plan2017high, plan2016generalized} include \cite{genzel2016high} and \cite{oymak2017fast}. In~\cite{genzel2016high}, the results for the $\calK$-Lasso are extended to a fairly large class of convex loss functions with the assumption that $\calK$ is convex. The authors of~\cite{oymak2017fast} develop a framework for characterizing time-data tradeoffs for various algorithms used to recover a structured signal from nonlinear observations.

For high-dimensional SIM with heavy-tailed elliptical symmetric measurements, \cite{goldstein2018structured, wei2018structured} propose thresholded least square estimators to attain similar performance guarantees to those for the Gaussian case. Thresholded score function estimators via Stein's identity are proposed in~\cite{yang2017stein, yang2019misspecified}, with the purpose of obtaining a consistent estimator for general non-Gaussian measurements. While treating heavy-tailed measurements, their methods depend heavily on the chosen basis, and appear difficult to generalize beyond sparse and low-rank signals.
 
Sharp error bounds (including constants) for generalized Lasso problems were provided in \cite{thrampoulidis2014simple, thrampoulidis2015lasso, oymak2013squared, oymak2013simple}.  Our focus is on scaling laws, and we leave refined studies of this kind for future work.

A table comparing our results with the most relevant existing works is given in Appendix \ref{app:table}.

\vspace*{-1ex}
\subsection{Contributions}\label{sec:contributions}
\vspace*{-1ex}

In this paper, we provide recovery guarantees for the $\calK$-Lasso with non-linear and corrupted observations under a generative prior. Our main results are outlined as follows: 
\begin{itemize}[leftmargin=5ex,itemsep=0ex,topsep=0.25ex]
    \item In Section~\ref{sec:main}, we characterize the number of measurements sufficient to attain a non-uniform and accurate recovery of an underlying signal in the range of a Lipschitz continuous generative model.   In Section~\ref{sec:nn}, we specialize this result to neural network generative models. 
    \item In Section~\ref{sec:extensions}, we discuss several variations or extensions of our main result, including an unknown covariance matrix for the random measurement vectors, relaxing the norm restriction for the underlying signal, and considering bounded $k$-sparse vectors. 
    \item  In Section~\ref{sec:uniform}, we provide uniform recovery guarantees under the assumption of a local embedding property, which holds for various models of interest.
\end{itemize}

\vspace*{-1ex}
\subsection{Notation}\label{sec:notations}
\vspace*{-1ex}

We use upper and lower case boldface letters to denote matrices and vectors respectively. We write $[N]=\{1,2,\cdots,N\}$ for a positive integer $N$. A {\em generative model} is a function $G \,:\, \calD\to \bbR^n$, with latent dimension $k$, ambient dimension $n$, and input domain $\calD \subseteq \bbR^k$. For a set $S \subseteq \bbR^k$ and a generative model $G \,:\,\bbR^k \to \bbR^n$, we write $G(S) = \{ G(\bz) \,:\, \bz \in S  \}$. We use $\|\bX\|_{2 \to 2}$ to denote the spectral norm of a matrix $\bX$. We define the $\ell_q$-ball $B_q^k(r):=\{\bz \in \bbR^k: \|\bz\|_q \le r\}$ for $q \in [0,+\infty]$, and we use $B_q^k$ to abbreviate $B_q^k(1)$. $\calS^{n-1} := \{\bx \in \bbR^n: \|\bx\|_2=1\}$ represents the unit sphere in $\bbR^n$.  The symbols $C,C',C'',c,c'$ are absolute constants whose values may differ from line to line.

\vspace*{-1ex}
\section{Problem Setup}\label{sec:prob_setup_examples}
\vspace*{-1ex}

In this section, we formally introduce the problem, and overview the main assumptions that we adopt. In addition, we provide examples of measurement models satisfying our assumptions. 

Before proceeding, we state the following standard definition.

\begin{definition} \label{def:subg}
 A random variable $X$ is said to be sub-Gaussian if there exists a positive constant $C$ such that $\left(\mathbb{E}\left[|X|^{p}\right]\right)^{1/p} \leq C  \sqrt{p}$ for all $p\geq 1$.  The sub-Gaussian norm of a sub-Gaussian random variable $X$ is defined as $\|X\|_{\psi_2}:=\sup_{p\ge 1} p^{-1/2}\left(\mathbb{E}\left[|X|^{p}\right]\right)^{1/p}$. 
\end{definition}

\vspace*{-1ex}
\subsection{Setup and Main Assumptions} \label{sec:prob_setup}
\vspace*{-1ex}

Recall that the (uncorrupted) measurement model is given in \eqref{eq:yi}, where the function $f(\cdot)$ may be random (but independent of $\bA$).  Except where stated otherwise, we make the following assumptions:
\begin{itemize}[leftmargin=5ex,itemsep=0ex,topsep=0.25ex]
 \item The measurement matrix $\bA$ has i.i.d.~standard Gaussian entries, i.e., $\ba_i \overset{i.i.d.}{\sim} \calN(\mathbf{0},\bI_n)$.
 \item The scaled vector $\mu \bx^*$ lies in a set of structured signals $\calK$, where $\mu$ is a fixed parameter depending on $f$ specified below. We focus on the case that $\calK = \mathrm{Range}(G)$ for some $L$-Lipschitz continuous generative model $G \,:\, B_2^k(r) \rightarrow \bbR^n$ (e.g., see \cite{bora2017compressed}).
 \item Similarly to \cite{plan2017high,genzel2018mismatch,sun2019recovery}, we assume that each $y_i$ is (unconditionally) sub-Gaussian. 
 \item In addition to any random noise in $f$, we allow for adversarial noise.  In this case, instead of observing $\by$ directly, we only assume access to $\tilde{\by} = [\tilde{y}_1,\ldots,\tilde{y}_m]^T \in \bbR^m$ satisfying
 \begin{equation}
    \frac{1}{\sqrt{m}} \|\tilde{\by}-\by\|_2 \le \tau
 \end{equation}
 for some parameter $\tau \ge 0$.  Note that the corruptions of $\by$ yielding $\tilde{\by}$ may depend on $\bA$.
 \item To derive an estimate of $\bx^*$ (up to constant scaling), we seek $\hat{\bx}$ minimizing $\|\tilde{\by} - \bA \bx\|_2$ over $\calK$:
 \begin{equation}\label{eq:gen_lasso}
  \hat{\bx} = \arg \min_{\bx \in \calK} \|\tilde{\by} -\bA\bx\|_2.
 \end{equation}
Recall that we refer to this generalized Lasso as the $\calK$-Lasso (with corrupted observations $\tilde{\by}$ and measurement matrix $\bA$) \cite{plan2016generalized,plan2017high}. It may seem counter-intuitive that the $\calK$-Lasso is provably accurate even for non-linear observations; the idea is that the nonlinearity is rather treated as noise and one may transform the non-linear observation model into a scaled linear model with an unconventional noise term~\cite{plan2016generalized}. 
 \item Let $g \sim \calN(0,1)$ be a standard normal random variable. To analyze the recovery performance as a function of the nonlinearity $f$, we use the following parameters, which play a key role:
 \begin{itemize}
  \item The mean term, denoted $\mu := \bbE [ f(g) g]$;
  \item The sub-Gaussian norm of $f(g)$ (i.e., of any given $y_i$), denoted $\psi := \|f(g)\|_{\psi_2}$. 
 \end{itemize}
\end{itemize}

\begin{remark}
    For $\calK = {\rm Range}(G)$, the estimator \eqref{eq:gen_lasso} was considered in \cite{bora2017compressed}, focusing on linear observations.  It was noted that although finding an exact solution may be difficult due to the (typical) non-convexity of $G$, gradient methods for finding an approximate solution are effective in practice.
\end{remark}

\vspace*{-1ex}
\subsection{Examples of Measurement Models}\label{sec:example}
\vspace*{-1ex}

When $f$ does not grow faster than linearly, i.e., $|f(x)| \le a+ b|x|$ for some scalars $a$ and $b$, $y_i$ will be sub-Gaussian. We may also consider various nonlinear models that give sub-Gaussian observations. For example, the censored Tobit model,  $f(x) = \max \{x,0\}$, gives $\mu = \frac{1}{2}$ and $\psi \le C$.

In addition, by setting $a = 1$ and $b=0$ in the above-mentioned condition $|f(x)| \le a+ b|x|$, we observe that measurement models with each output selected from $\{-1,1\}$, i.e., 1-bit measurements, lead to sub-Gaussian $y_i$. For example, for 1-bit observations with random bit flips, we set $f(x) = \xi \cdot \mathrm{sign}(x)$, where $\xi$ is an independent $\pm 1$-valued random variable with $\bbP(\xi = -1) = p < \frac{1}{2}$. In this case, we have $\mu = (1-2p)\sqrt{\frac{2}{\pi}}$ and $\psi = 1$ \cite{genzel2016high}.

We may also consider additive noise before 1-bit quantization, i.e., $f(x) := \mathrm{sign}(x + z)$, where $z$ is an independent noise term. Different forms of noise lead to distinct binary statistical models. For example, if $z$ is Gaussian, this corresponds to the probit model. On the other hand, if $z$ is logit noise, this recovers the logistic regression model. 
More generally, we may consider non-binary quantization schemes, such as uniform (mid-riser) quantization: For some $\Delta >0$, $f(x) = \Delta \left(\floor{\frac{x}{\Delta}}+ \frac{1}{2}\right)$. It is easy to see that $|f(x)| \le |x| + \frac{\Delta}{2}$ for all $x \in \bbR$, and thus the corresponding observations are sub-Gaussian. In addition, we have $\mu = 1$ and $\psi \le C + \frac{\Delta}{2}$ \cite{thrampoulidis2020generalized}.

It is also worth noting that certain popular models, such as phase retrieval~\cite{fienup1982phase,candes2015phase} with $f(x) = x^2$, do not satisfy the sub-Gaussianity assumption.  Such models are beyond the scope of this work, but could potentially be handled via analogous ideas to \cite{plan2016generalized}.

\vspace*{-1ex}
\section{Main Result}\label{sec:main}
\vspace*{-1ex}

In the following, we state our main theorem concerning non-uniform recovery, i.e., the vector $\bx^*$ is fixed in advance, before the sample matrix $\bA$ is drawn. 

\begin{theorem}\label{thm:generative}
 Consider any $\bx^* \in \calS^{n-1} \cap \frac{1}{\mu}\calK$ with $\calK = G(B_2^k(r))$ for some $L$-Lipschitz $G \,:\, B_2^k(r) \to \bbR^n$, along with $\by$ from the model \eqref{eq:yi} with $\ba_i \overset{i.i.d.}{\sim} \calN(\mathbf{0},\bI_n)$, and an arbitrary corrupted vector $\tilde{\by}$ with $\frac{1}{\sqrt{m}} \|\tilde{\by}-\by\|_2 \le \tau$.
 For any $\epsilon >0$, if $Lr = \Omega(\epsilon \psi n)$ and $m = \Omega\big(\frac{k}{\epsilon^2}\log \frac{Lr}{\epsilon \psi}\big)$,\footnote{Here and in subsequent results, the implied constants in these $\Omega(\cdot)$ terms are implicitly assumed to be sufficiently large.  Regarding the assumption $Lr = \Omega(\epsilon \psi n)$, we note that $d$-layer neural networks typically give $L = n^{\Theta(d)}$~\cite{bora2017compressed}, meaning this assumption is certainly satisfied for fixed $(r,\psi,\epsilon)$.} then with probability $1-e^{-\Omega(\epsilon^2 m)}$, any solution $\hat{\bx}$ to the $\calK$-Lasso~\eqref{eq:gen_lasso} satisfies 
 \begin{equation}
  \|\mu \bx^* - \hat{\bx}\|_2 \le \psi\epsilon + \tau.
 \end{equation}
\end{theorem}

If $\tau = 0$ and $\psi > 0$ is fixed, we get an error bound on the order of $\sqrt{\frac{k\log (L r)}{m}}$ up to a logarithmic factor in $m$, in particular matching the usual $m^{-\frac{1}{2}}$ scaling.  In addition, the $k \log (Lr)$ dependence (as well as the effect of $\tau > 0$) is consistent with prior work on the linear \cite{bora2017compressed} and 1-bit \cite{liu2020sample} models, while also holding for broader non-linear models.  Additional variations are given in Section \ref{sec:extensions}, and a uniform guarantee is established in Section \ref{sec:uniform} under additional assumptions.  

\subsection{Proof of Theorem \ref{thm:generative}}

Before presenting the proof of Theorem~\ref{thm:generative}, we state some useful auxiliary results. First, we present the definition of the Set-Restricted Eigenvalue Condition (S-REC)~\cite{bora2017compressed}, which generalizes the REC. 
\begin{definition}
 Let $S \subseteq \bbR^n$. For parameters $\gamma >0$, $\delta \ge 0$, a matrix $\tilde{\bA} \in \bbR^{m \times n}$ is said to satisfy the S-REC($S,\gamma,\delta$) if, for every $\bx_1,\bx_2 \in S$, it holds that
 \begin{equation}
  \|\tilde{\bA}(\bx_1 -\bx_2)\|_2 \ge \gamma \|\bx_1 -\bx_2\|_2 -\delta.
 \end{equation}
\end{definition}

Recalling that $\bA \in \bbR^{m \times n}$ has i.i.d.~$\calN(0,1)$ entries, the following lemma from \cite{bora2017compressed} shows that $\frac{1}{\sqrt{m}}\bA$ satisfies the S-REC condition for bounded Lipschitz generative models.
    \begin{lemma}{\em (\hspace{1sp}\cite[Lemma~4.1]{bora2017compressed})}\label{lem:boraSREC}
     Fix $r > 0$, and let $G \,:\, B_2^k(r) \rightarrow \bbR^n$ be $L$-Lipschitz. For $\alpha \in (0,1)$, if $m = \Omega\left(\frac{k}{\alpha^2} \log \frac{Lr}{\delta}\right)$, 
     then a random matrix $\frac{1}{\sqrt{m}}\bA \in \bbR^{m \times n}$ with $a_{ij} \overset{i.i.d.}{\sim} \calN\left(0,1\right)$ satisfies the S-REC$(G(B_2^k(r),1-\alpha,\delta))$ with probability $1-e^{-\Omega(\alpha^2 m)}$.
\end{lemma}

Using a basic two-sided concentration bound for standard Gaussian matrices ({\em cf.}, Lemma~\ref{lem:norm_pres} in Appendix~\ref{sec:preliminary}), and a simple modification of the proof of Lemma~\ref{lem:boraSREC} given in \cite{bora2017compressed}, we obtain the following lemma, which is useful for upper bounding the error corresponding to adversarial noise.

\begin{lemma}\label{lem:boraSREC_gen}
    Fix $r > 0$, and let $G \,:\, B_2^k(r) \rightarrow \bbR^n$ be $L$-Lipschitz.
 For $\alpha <1$ and $\delta >0$, if $m = \Omega\left(\frac{k}{\alpha^2} \log \frac{Lr}{\delta}\right)$, then with probability $1-e^{-\Omega(\alpha^2 m)}$, we have for all $\bx_1,\bx_2 \in G(B_2^k(r))$ that 
 \begin{equation}
  \frac{1}{\sqrt{m}}\|\bA \bx_1 - \bA \bx_2\|_2 \le (1+\alpha) \|\bx_1 -\bx_2\|_2 +\delta.
 \end{equation}
\end{lemma}

Using a chaining argument similar to~\cite{bora2017compressed,liu2020sample}, we additionally establish the following technical lemma, whose proof is given in Appendix~\ref{sec:preliminary}. Note that here we only require that $\tilde{\bx} \in \calK = G(B_2^k(r))$, and we do not require that $\mu \bar{\bx} \in \calK$.

\begin{lemma}\label{lem:bxStarbxHat}
 Fix any $\bar{\bx} \in \calS^{n-1}$ and let $\bar{\by} := f(\bA \bar{\bx})$. Suppose that some $\tilde{\bx} \in \calK$ is selected depending on $\bar{\by}$ and $\bA$.\footnote{For example, we may choose $\tilde{\bx}$ to be a minimizer of the $\calK$-Lasso~\eqref{eq:gen_lasso} with inputs $\bar{\by}$ and $\bA$.} For any $\delta >0$, if $Lr = \Omega(\delta n)$ and $m = \Omega\left(k \log \frac{Lr}{\delta}\right)$, then with probability $1-e^{-\Omega\left(k \log \frac{Lr}{\delta}\right)}$, it holds that
 \begin{equation}
  \left\langle \frac{1}{m}\bA^T(\bar{\by} - \mu\bA \bar{\bx}), \tilde{\bx}-\mu \bar{\bx} \right \rangle \le O\left(\psi \sqrt{\frac{k\log \frac{Lr}{\delta}}{m}}\right)\|\tilde{\bx} - \mu\bar{\bx}\|_2 + O\left(\delta\psi\sqrt{\frac{k\log \frac{Lr}{\delta}}{m}}\right).
 \end{equation}
\end{lemma}

With the above auxiliary results in place, the proof of Theorem~\ref{thm:generative} is given as follows.

\begin{proof}[Proof of Theorem~\ref{thm:generative}]
Because $\hat{\bx}$ is a solution to the $\calK$-Lasso and $\mu \bx^* \in \calK$, we have 
\begin{equation}
 \|\tilde{\by} - \mu\bA \bx^*\|_2^2  \ge \|\tilde{\by} - \bA \hat{\bx}\|_2^2 = \left\|(\tilde{\by} - \mu\bA \bx^*) - \bA\left(\hat{\bx}-\mu \bx^*\right)\right\|_2^2,
\end{equation}
and expanding the square and diving by $m$ gives 
\begin{align}
 \frac{1}{m}\|\bA(\hat{\bx}-\mu\bx^*)\|_2^2 
    &\le \frac{2}{m} \langle \tilde{\by} -\mu\bA \bx^*, \bA( \hat{\bx}- \mu \bx^*) \rangle \\
    &= \frac{2}{m}\langle \bA^T (\tilde{\by} -\mu\bA\bx^*),\hat{\bx}- \mu\bx^* \rangle. \label{eq:basic_genLasso}
\end{align}
We aim to derive a suitable lower bound on $\frac{1}{m}\|\bA(\hat{\bx}-\mu\bx^*)\|_2^2$ and an upper bound on $\frac{2}{m}\langle \bA^T (\tilde{\by} -\mu\bA\bx^*),\hat{\bx}- \mu\bx^* \rangle$.  For any $\delta \in (0,1)$ satisfying $Lr = \Omega(\delta n)$ (to be verified later), setting $\alpha  = \frac{1}{2}$ in Lemma~\ref{lem:boraSREC}, we have that if $m = \Omega\left(k \log \frac{Lr}{\delta}\right)$, then with probability $1-e^{-\Omega(m)}$, 
\begin{equation}
 \left\|\frac{1}{\sqrt{m}} \bA (\hat{\bx}-\mu\bx^*)\right\|_2 \ge \frac{1}{2}\|\hat{\bx}-\mu\bx^*\|_2 -\delta. 
\end{equation}
Taking the square on both sides and re-arranging, we obtain 
\begin{equation}\label{eq:lb_lhs}
 \frac{1}{4}\|\hat{\bx}-\mu\bx^*\|_2^2 \le  \frac{1}{m}\left\| \bA (\hat{\bx}-\mu\bx^*)\right\|_2^2 + \delta \|\hat{\bx}-\mu\bx^*\|_2. 
\end{equation}
Recalling that $\by = f(\bA\bx^*)$, we also have
\begin{equation}\label{eq:main_twoterms}
 \frac{1}{m}\langle \bA^T (\tilde{\by} -\mu\bA\bx^*),\hat{\bx}- \mu\bx^* \rangle = \frac{1}{m}\langle \bA^T (\tilde{\by} -\by),\hat{\bx}- \mu\bx^* \rangle + \frac{1}{m}\langle \bA^T (\by -\mu\bA\bx^*),\hat{\bx}- \mu\bx^* \rangle.
\end{equation}
To bound the first term, note that using $\alpha = \frac{1}{2}$ in Lemma~\ref{lem:boraSREC_gen}, we have with probability $1-e^{-\Omega(m)}$ that 
 \begin{equation}
  \left\|\frac{1}{\sqrt{m}} \bA (\hat{\bx} - \mu\bx^*)\right\|_2 \le O(\|\hat{\bx} - \mu\bx^*\|_2 + \delta). \label{eq:l2bound}
 \end{equation}
Therefore, the following holds with probability $1-e^{-\Omega(m)}$: 
    \begin{align}
     \left\langle \frac{1}{m} \bA^T (\by - \tilde{\by}), \hat{\bx} - \mu\bx^* \right\rangle 
    &= \left\langle \frac{1}{\sqrt m} (\by - \tilde{\by}),  \frac{1}{\sqrt m}  \bA(\hat{\bx} - \mu\bx^*) \right\rangle \\
    & \le \left\|\frac{1}{\sqrt{m}}(\by - \tilde{\by})\right\|_2 \cdot \left\|\frac{1}{\sqrt{m}} \bA (\hat{\bx} - \mu\bx^*)\right\|_2 \\ 
     & \le \tau O(\|\hat{\bx} - \mu\bx^*\|_2 + \delta) \label{eq:main_first_term}
    \end{align}
by \eqref{eq:l2bound} and the assumption $\frac{1}{\sqrt{m}} \|\tilde{\by}-\by\|_2 \le \tau$. 

We now consider the second term in \eqref{eq:main_twoterms}.  From Lemma~\ref{lem:bxStarbxHat}, we have that when $Lr = \Omega(\delta n)$ and $m = \Omega\left(k\log\frac{Lr}{\delta}\right)$, with probability $1-e^{-\Omega\left(k\log \frac{Lr}{\delta}\right)}$,
\begin{equation}\label{eq:beUsefulForUnion}
 \left\langle \frac{1}{m}\bA^T(\by - \mu\bA \bx^*), \hat{\bx}-\mu \bx^* \right \rangle \le O\left(\psi\sqrt{\frac{k\log \frac{Lr}{\delta}}{m}}\right)\|\hat{\bx} - \mu\bx^*\|_2 + O\left(\psi\delta\sqrt{\frac{k\log \frac{Lr}{\delta}}{m}}\right).
\end{equation}

Putting the preceding findings together, we have the following with probability $1-e^{-\Omega\left(k\log \frac{Lr}{\delta}\right)}$: 
\begin{align}
 \|\mu \bx^* - \hat{\bx}\|_2^2 
 &\le \frac{4}{m}\|\bA (\hat{\bx}-\mu \bx^*)\|_2^2 + 4\delta \|\hat{\bx}-\mu \bx^*\|_2 \label{eq:non_unif_final0} \\
 & \le 8 \left\langle \frac{1}{m}\bA^T(\tilde{\by} - \mu\bA \bx^*), \hat{\bx}-\mu \bx^* \right \rangle + 4 \delta \|\hat{\bx}-\mu \bx^*\|_2 \label{eq:non_unif_final1} \\
 & \le O\left(\psi\sqrt{\frac{k\log \frac{Lr}{\delta}}{m}} + \delta + \tau \right)\|\hat{\bx} - \mu\bx^*\|_2 + O\left(\tau \delta + \psi\delta\sqrt{\frac{k\log \frac{Lr}{\delta}}{m}}\right),\label{eq:non_unif_final}
\end{align}
where \eqref{eq:non_unif_final0} uses \eqref{eq:lb_lhs}, \eqref{eq:non_unif_final1} uses \eqref{eq:basic_genLasso}, and  \eqref{eq:non_unif_final} combines \eqref{eq:main_twoterms},~\eqref{eq:main_first_term} and~\eqref{eq:beUsefulForUnion}.

By considering both possible cases of which of the two terms in \eqref{eq:non_unif_final} is larger, we find that if
\begin{equation}
 \delta = O\left(\psi\sqrt{\frac{k\log \frac{Lr}{\delta}}{m}}\right),
\end{equation}
then we have
\begin{equation}
 \|\mu \bx^* - \hat{\bx}\|_2 \le O\left(\psi\sqrt{\frac{k\log \frac{Lr}{\delta}}{m}} + \tau\right).
\end{equation}
Then, for any $\epsilon \in (0,1)$ satisfying $Lr =\Omega(\epsilon \psi n)$ (as assumed in the theorem), setting $m = \Omega\left(\frac{k}{\epsilon^2}\log \frac{Lr}{\delta}\right)$ leads to $\delta = O\big(\psi\sqrt{\frac{k\log \frac{Lr}{\delta}}{m}}\big) = O(\epsilon\psi)$ and thus $Lr = \Omega(\delta n)$ (as assumed previously in the proof).  Hence, we have with probability $1-e^{-\Omega(\epsilon^2 m)}$ that 
\begin{equation}
\|\mu\bx^* - \hat{\bx}\|_2 \le \epsilon \psi + \tau.
\end{equation}
\end{proof}

\vspace*{-1ex}
\subsection{Application to Neural Network Generative Models}\label{sec:nn}
\vspace*{-1ex}

In the following, we apply Theorem \ref{thm:generative} to neural network models, as these are of particular practical interest.  We consider feedforward neural network generative models; with $d$ layers, we have
\begin{equation}\label{eq:nn_function}
 G(\bz) = \phi_d\left(\phi_{d-1}\left(\cdots \phi_2( \phi_1(\bz,\btheta_1), \btheta_2)\cdots, \btheta_{d-1}\right), \btheta_d\right),
\end{equation}
where $\bz \in B_2^k(r)$, $\phi_i(\cdot)$ is the functional mapping corresponding to the $i$-th layer, and $\btheta_i = (\bW_i,\bb_i)$ is the parameter pair for the $i$-th layer:  $\bW_i \in \bbR^{n_i \times n_{i-1}}$ is the matrix of weights, and $\bb_i \in \bbR^{n_i}$ is the vector of offsets, where $n_i$ is the number of neurons in the $i$-th layer.  Note that $n_0 = k$ and $n_d = n$. Defining $\bz^0 = \bz$ and $\bz^i= \phi_i(\bz^{i-1},\btheta_i)$, we set $\phi_i(\bz^{i-1},\btheta_i) = \phi_i(\bW_i \bz^{i-1}+ \bb_i)$, $i = 1,2,\ldots,d$, for some activation function $\phi_i(\cdot)$ applied element-wise.

The following corollary applies Theorem~\ref{thm:generative} to feedforward neural network generative models. Note that here we do not constrain the $\ell_2$-norm of the signal $G(\bz^*) \in \calK$.

\begin{corollary}\label{coro:main_nn}
     Suppose that the generative model $G\,:\, B_2^k(r) \rightarrow \bbR^n$ is defined as in~\eqref{eq:nn_function} with at most $w$ nodes per layer. Suppose that all weights are upper bounded by $W_{\max}$ in absolute value, and that the activation function is $1$-Lipschitz. For any $\bz^* \in B_2^k(r)$, let $\by = f\big(\bA \frac{G(\bz^*)}{\mu}\big)$ and let $\tilde{\by}$ be the observed vector with $\frac{1}{\sqrt{m}}\|\by - \tilde{\by}\|_2 \le \tau$.  In addition, define $\bar{f}(x) = f\big(\frac{\|G(\bz^*)\|_2}{\mu} x\big)$ and $\bar{\mu} = \bbE[\bar{f}(g)g]$, $\bar{\psi}=\|\bar{f}(g)\|_{\psi_2}$. Then, for any $\epsilon  > 0$, if $(wW_{\max})^d r = \Omega(\epsilon \bar{\psi} n)$, $m = \Omega\big(\frac{k}{\epsilon^2}\log \frac{  r(wW_{\max})^d }{\epsilon\bar{\psi}}\big)$ and $\frac{\bar{\mu} G(\bz^*)}{\|G(\bz^*)\|_2} \in \calK$, then with probability  $1-e^{-\Omega(\epsilon^2 m)}$, any solution $\hat{\bx}$ to the $\calK$-Lasso~\eqref{eq:gen_lasso} satisfies
     \begin{equation}
      \left\|\bar{\mu} \frac{G(\bz^*)}{\|G(\bz^*)\|_2} - \hat{\bx}\right\|_2 \le \epsilon\bar{\psi} + \tau.
     \end{equation}
\end{corollary}
\begin{proof}
We know that under the assumptions of the corollary, the generative model $G$ is $L$-Lipschitz with $L = \left(w W_{\max}\right)^d$ ({\em cf.}~\cite[Lemma~8.5]{bora2017compressed}). Letting $\rho = \|G(\bz^*)\|_2$, it is straightforward to see that $\bar{f}(g)=f\big(\frac{\rho}{\mu} g \big)$ is also sub-Gaussian, where $g\sim \calN(0,1)$. In addition, we have 
 \begin{equation}
  \by = f\left(\bA \frac{G\left(\bz^*\right)}{\mu}\right) = f\left(\bA \frac{\rho}{\mu}\cdot \frac{G\left(\bz^*\right)}{\rho}\right) = \bar{f}\left(\bA \frac{G\left(\bz^*\right)}{\rho}\right).
 \end{equation}
Note that $\frac{G\left(\bz^*\right)}{\rho}$ is a unit vector, and $\bar{\mu} \frac{G\left(\bz^*\right)}{\rho} \in \calK$ by assumption. Applying Theorem~\ref{thm:generative} to the observation function $\bar{f}$ and the unit signal vector $\frac{G\left(\bz^*\right)}{\rho}$ completes the proof.
\end{proof}

Several commonly-used activation functions are $1$-Lipschitz, such as i) the ReLU function, $\phi_i(x) = \max(x,0)$; (ii) the Sigmoid function, $\phi_i(x) = \frac{1}{1+e^{-x}}$; and (iii) the Hyperbolic tangent function with $\phi_i(x) = \frac{e^{x}-e^{-x}}{e^{x}+e^{-x}}$.  Moreover, it is straightforward to generalize to other activation functions whose Lipschitz constants may exceed one.  

 The assumptions in Corollary~\ref{coro:main_nn} pose some limitations, but are satisfied in several cases of interest.  For example, the assumption $\frac{\bar{\mu} G(\bz^*)}{\|G(\bz^*)\|_2} \in \calK$ is satisfied when the generative model is a ReLU network with no offsets (see~\cite[Remark 2.1]{wei2019statistical}), due to $\calK$ being cone-shaped. In addition, while the sub-Gaussianity constant $\bar{\psi}=\|\bar{f}(g)\|_{\psi_2}$ is dependent on $\bz^*$, it can be upper bounded independently of $\bz^*$ under any observation model in which the measurements are uniformly bounded (e.g., including not only 1-bit, but also more general multi-bit quantized models).

\vspace*{-1ex}
\section{Variations and Extensions}\label{sec:extensions}
\vspace*{-1ex}
In this section, we discuss several variations and extensions of our main result, including considering bounded $k$-sparse vectors in Section~\ref{sec:bddSparse}, 
an unknown covariance matrix for the random measurement vectors in Section~\ref{sec:general_cov}, relaxing the norm restriction for the underlying signal in Section~\ref{sec:removingNormAssump}.  Some additional variations are given in the appendices, namely, guarantees for a distinct correlation-based optimization algorithm under binary observations (Appendix \ref{app:alterBinary}), and connections between our sample complexity and the Gaussian mean width (Appendix \ref{app:relateGMW}).

\vspace*{-0.5ex}
\subsection{Bounded Sparse Vectors}\label{sec:bddSparse}
\vspace*{-1ex}

In the proof of Theorem~\ref{thm:generative}, for the set of signals $\calK = G(B_2^k(r))$, we make use of the property that for any $\delta>0$, there exists a $\delta$-net $\calM$ of $\calK$ such that $|\calM| \le O\left(\exp\left(k \log \frac{Lr}{\delta}\right)\right)$.  Hence, we can readily extend the result to other sets $\calK$ with known bounds on the size of a $\delta$-net.  As an example, we state the following for bounded sparse vectors, defining $\Sigma_k^n$ to be the set of $k$-sparse vectors in $\bbR^n$.  A proof outline is given in Appendix \ref{sec:pf_sparse}.

\begin{corollary}\label{coro:bddSparse}
 Fix $\epsilon >0$, and let $\nu \ge \mu$ satisfy $\nu = \Omega(\epsilon \psi k)$. Fix $\bx^* \in \Sigma_k^n \cap \calS^n$, let $\by = f(\bA \bx^*)$, and let $\tilde{\by}$ be a vector satisfying $\frac{1}{\sqrt{m}}\|\by -\tilde{\by}\|_2 \le \tau$. Then, when $m = \Omega\big(\frac{k}{\epsilon^2}\log \frac{\nu n}{\epsilon \psi k}\big)$, with probability  $1-e^{-\Omega(\epsilon^2 m)}$, any $\hat{\bx}$ that minimizes $\|\tilde{\by}-\bA \bx\|_2$ over $\Sigma_{k}^n \cap \nu B_2^n$ satisfies 
 \begin{equation}
  \|\mu \bx^* - \hat{\bx}\|_2 \le \psi\epsilon + \tau.
 \end{equation}
\end{corollary}

This corollary is similar to other sparsity based results for the generalized Lasso, such as those in~\cite{plan2016generalized,plan2017high}.  It is intuitive that similar sparsity-based results to Theorem \ref{thm:generative} follow without difficulty, given that generative models are known that can produce bounded sparse signals \cite{liu2020information,kamath2020power}.

\vspace*{-1ex}
\subsection{General Covariance Matrices} \label{sec:general_cov}
\vspace*{-1ex}

Thus far, we have focused on the case that $\ba_i \sim \calN(\mathbf{0},\bI)$.  Following the ideas of \cite{plan2016generalized}, we can also consider the more general scenario in which $\ba_i \sim \calN(\mathbf{0},\bSigma)$ for an {\em unknown} covariance matrix $\bSigma \in \bbR^{n \times n}$, assuming that $\|\sqrt{\bSigma}\bx^*\|_2 =1$ and $\mu \bx^* \in \calK$. The definitions of $\mu$ and $\psi$ remain the same, {\em cf.}, Section~\ref{sec:prob_setup}. The following is easily deduced from Theorem \ref{thm:generative}; see Appendix \ref{sec:pf_gen_cov} for the details.

\begin{corollary}\label{coro:arb_cov}
 Suppose that $\ba_i \overset{i.i.d.}{\sim} \calN(\mathbf{0},\bSigma)$ for $i \in [m]$ and $\|\sqrt{\bSigma}\bx^*\|_2 =1$. Suppose that $\by = f(\bA \bx^*)$ and $\mu \bx^* \in \calK$. Let $\tilde{\by}$ be any vector of corrupted measurements satisfying $\frac{1}{\sqrt{m}}\|\by - \tilde{\by}\|_2 \le \tau$. Then, for any $\epsilon >0$, when $\|\bSigma\|_{2\to 2}^{1/2} Lr = \Omega(\epsilon \psi n)$ and $m = \Omega\big(\frac{k}{\epsilon^2} \log \frac{\|\bSigma\|_{2\to 2}^{1/2} Lr}{\epsilon \psi}\big)$, with probability $1-e^{-\Omega(\epsilon^2 m)}$, any solution to the generalized Lasso~\eqref{eq:gen_lasso} satisfies 
 \begin{equation}
  \|\sqrt{\bSigma}(\hat{\bx}-\mu\bx^*)\|_2 \le \psi \epsilon + \tau.
 \end{equation}
\end{corollary}

\vspace*{-1ex}
\subsection{Removing the $\ell_2$-norm Assumption}\label{sec:removingNormAssump} \vspace*{-1ex}

Continuing from the previous subsection and again following \cite{plan2016generalized}, our results can easily be generalized to the case that $\|\sqrt{\bSigma} \bx^*\|_2 \ne 1$ (or for $\bSigma = \bI$, the case that $\|\bx^*\|_2 \ne 1$). The idea is similar to that presented in the proof of Corollary~\ref{coro:main_nn}. In particular, setting $\rho = \|\sqrt{\bSigma} \bx^*\|_2$ and $\bar{\bx} = \frac{\bx^*}{\rho}$ gives
\begin{equation}\label{eq:wrongEq}
 f(\bA \bx^*) = f(\rho \bA \bar{\bx})= \bar{f}(\bA \bar{\bx}),
\end{equation}
where $ \bar{f}(x) := f(\rho x)$ for $x \in \bbR$. Hence, for $g \sim \calN(0,1)$, if $\bbE[\bar{f}(g)g] \bar{\bx} \in \calK$, the preceding theorems and corollaries apply to the estimation of $\bar{\bx}$, with modified parameters
\begin{equation}
 \bar{\mu} := \bbE[\bar{f}(g)g], \quad \bar{\psi} := \|\bar{f}(g)\|_{\psi_2}. 
\end{equation}

\vspace*{-1ex}
\section{Uniform Recovery Guarantees}\label{sec:uniform}
\vspace*{-1ex}

In this section, we turn to uniform recovery guarantees, stating that a single matrix $\bA$ simultaneously permits the recovery of all $\bx^*$ in the set of interest.  For brevity, we consider $\mu$ and $\psi$ to be fixed constants and omit them in the $O(\cdot)$ notation.  

Our result will depend on the following Local Embedding Property (LEP).

\begin{definition} \label{def:lep}
    A deterministic function $\tilde{f}$  and measurement matrix $\tilde{\bA}$ are said to satisfy the LEP$(S,\delta, \beta)$ with set $S$ and parameters $\delta \ge 0$ and $\beta \ge 0$ if, for any $\bx_1 \in S$ and $\bx_2 \in S$ satisfying $\|\bx_1 - \bx_2\|_2 \le \delta$, the following holds:
    \begin{equation}
        \frac{1}{\sqrt{m}}\|\tilde{f}(\tilde{\bA}\bx_1)-\tilde{f}(\tilde{\bA}\bx_2)\|_2 \le C \delta^\beta \label{eq:lep}
    \end{equation}  
    for some $C > 0$ not depending on $\delta$.
\end{definition}

This definition essentially states that nearby signals remain close upon multiplying by $\tilde{\bA}$ and then applying the function $\tilde{f}$.  See, for example, \cite{oymak2015near,liu2020sample} for similar concepts in earlier works. With this definition in place, our main assumption in this section is stated as follows.

\begin{assumption} \label{as:lep}
    Under the (possibly random) function $f$, i.i.d.~Gaussian measurement matrix $\bA$, and generative model $G$ with $\calK = {\rm Range}(G)$, there exists a constant $\beta \in (0,1]$ and functions $\funcM, \funcP$ such that for any sufficiently small $\delta$, the following holds with probability $1-\funcP(\delta,\beta)$ when $m \ge \funcM(\delta,\beta)$: The pair $(f,\bA)$ satisfies the LEP$(S,\delta,\beta)$ with $S = \calS^{n-1} \cap \{ \bx\,:\, c\bx \in \calK {\rm ~for~some~ }c \in [\mu (1- \eta),\mu(1 + \eta)]\}$, where $\eta > 0$ is a (small) positive constant not depending on $\delta$, and $\mu = \bbE[ f(g) g ]$.
\end{assumption}

While Assumption~\ref{as:lep} is somewhat technical, the intuition behind it is simply that if $\bx_1$ is close to $\bx_2$, then $f(\bA\bx_1)$ is close to $f(\bA\bx_2)$. We restrict $\beta \le 1$ because the case $\beta > 1$ fails even for linear measurements, and the LEP for $\beta > 1$ implies the same for $\beta = 1$. 
Before providing some examples of models satisfying Assumption \ref{as:lep}, we state our uniform recovery result, proved in Appendix~\ref{app:pf_uniform}.

\begin{theorem}\label{thm:generative_uniform}
    Suppose that $f$ yields parameters $\mu = \Theta(1)$ and $\psi = \Theta(1)$, and that Assumption \ref{as:lep} holds.  Then, for sufficiently small $\epsilon >0$, if $Lr = \Omega(\epsilon n)$ and $m \ge \funcM\left(\calK,\epsilon^{1/\beta},\beta\right) + \Omega\big(\frac{k}{\epsilon^2}\log \frac{Lr}{\epsilon} \big)$, then with probability $1-e^{-\Omega\left(m\right)}-\funcP\left(\calK,\epsilon^{1/\beta},\beta\right)$, we have the following: For any signal $\bx^* \in \calS^{n-1}$ with $\mu \bx^* \in \calK$ and $\by = f(\bA\bx^*)$, and any vector $\tilde{\by}$ of corrupted measurement satisfying $\frac{1}{\sqrt{m}}\|\tilde{\by} - \by\|_2 \le \tau$, any solution $\hat{\bx}$ to the $\calK$-Lasso satisfies 
 \begin{equation}
  \|\mu\bx^* - \hat{\bx}\|_2 \le \epsilon + \tau.
 \end{equation}
\end{theorem}

Assumption \ref{as:lep} is satisfied by various measurement models; for example:
\begin{itemize}[leftmargin=5ex,itemsep=0ex,topsep=0.25ex]
    \item Under the linear model $f(x) = x$, setting $\alpha = \frac{1}{2}$ in Lemma~\ref{lem:boraSREC_gen}, choosing $\delta >0$, and setting $\beta =1$ and $\mu =1$, we obtain $\funcM(\delta,\beta) = O\big(k \log \frac{Lr}{\delta}\big)$ and $\funcP(\delta,\beta) = e^{-\Omega(m)}$. 
    \item The preceding example directly extends to any $1$-Lipschitz function $f$, such as the censored Tobit model with $f(x) = \max\{x,0\}$.
    \item In Appendix \ref{app:lep}, we use an existing result in \cite{liu2020sample} to show that for the noiseless 1-bit model with $f(x) = \sign(x)$, we can choose any $\delta = O(1)$, set $\beta = \frac{1}{2}$ and $\mu = \sqrt{\frac{2}{\pi}}$, and obtain $\funcM(\delta,\beta)=O\big(\frac{k}{\delta}\log \frac{L r}{\delta^2}\big)$ and $\funcP(\delta,\beta) =e^{-\Omega(\delta m)}$.
\end{itemize}
Regarding the last of these, we note that our sample complexity in Theorem \ref{thm:generative_uniform} matches that of \cite[Corollary~3]{liu2020sample}.  An advantage of our result compared to \cite{liu2020sample} is that the $\calK$-Lasso objective function can be optimized directly using gradient methods, whereas the Hamming distance based objective proposed in \cite{liu2020sample} appears to be difficult to use directly in practice.  Instead, it is proposed in \cite{liu2020sample} to first approximate the objective by a convex one, and then apply a sub-gradient based method.


\vspace*{-1ex}
\section{Conclusion}
\vspace*{-1ex}

We have provided recovery guarantees for the generalized Lasso with nonlinear observations and generative priors. In particular, we showed that under i.i.d. Gaussian measurements, roughly $O\left(\frac{k}{\epsilon^2}\log L\right)$ samples suffice for non-uniform $\epsilon$-recovery, with robustness to adversarial noise. Moreover, we derived a uniform recovery guarantee under the assumption of the local embedding property. Possible extensions for future work include handling signals with representation error (i.e., $\mu\bx^*$ is not quite in $\calK$) \cite{bora2017compressed,plan2016generalized}, a sharp analysis including constants \cite{thrampoulidis2014simple, thrampoulidis2015lasso, oymak2013squared, oymak2013simple}, and lower bounds on the sample complexity \cite{plan2017high,kamath2020power,liu2020information}.

\smallskip
{\bf Acknowledgment.} This work was supported by the Singapore National Research Foundation (NRF) under grant number R-252-000-A74-281.

\newpage

\vspace*{-1ex}
\section*{Broader Impact}
\vspace*{-1ex}

{\bf Who may benefit from this research.} This is a theory paper primarily targeted at the research community.  The signal recovery techniques studied could potentially be useful for practitioners in areas such as image processing, audio processing, and medical imaging.

{\bf Who may be put at disadvantage from this research.} We are not aware of any significant/imminent risks of placing anyone at a disadvantage.

{\bf Consequences of failure of the system.} We believe that most failures should be immediately evident and detectable due to visibly poor reconstruction performance, and any such outputs could be discarded as needed.  However, some more subtle issues could arise, such as the reconstruction missing important details in the signal due to the generative model not capturing them.  As a result, care is advised in the choice of generative model, particularly in applications for which the reconstruction of fine details is crucial.

{\bf Potential biases.} The signal recovery algorithm that we consider takes as input an arbitrary pre-trained generative model.  If such a pre-trained model has inherent biases, they could be transferred to the signal recovery algorithm.

\bibliographystyle{myIEEEtran}
\bibliography{writeups,JS_References}

\newpage
\appendix

{\centering
    {\huge \bf Supplementary Material}
    
    {\Large \bf The Generalized Lasso with Nonlinear Observations \\ and Generative Priors (NeurIPS 2020) \par }  

    {\large \bf Zhaoqiang Liu and Jonathan Scarlett \par }  
}

\bigskip

\vspace*{-1ex}
\section{Table Comparing to the Existing Literature} \label{app:table}
\vspace*{-1ex}

The comparison of our results to those in the existing literature, as discussed in Section \ref{sec:related}, is outlined in Table \ref{table:first}.  In the table, we write $\mu = \bbE[f(g)g]$ for $g \sim \calN(0,1)$. We use $\calK$ to represent the structured set of interest, and $\Sigma_{k}^n$ to represent the set of $k$-sparse vectors in $\bbR^n$. For Projected Back Projection (PBP) \cite{plan2017high}, the reconstructed vector is $\hat{\bx}:=\calP_{\calK}\left(\frac{1}{m}\bA^T \by\right)$, where $\calP_\calK$ is the projection operator onto $\calK$. In addition, $\partial \calK$ represents the boundary of $\calK$. Letting $q\,:\, \bbR^n \rightarrow \bbR$ be the density of the random measurement vector $\ba$ and assume that $q$ is differentiable, we write $S_q (\ba) = -\frac{\nabla q(\ba)}{q(\ba)}$. For thresholded Empirical Risk Minimization (ERM), the reconstructed vector is $\hat{\bx}:=\arg \min_{\bx \in G(B_2^k(r))} \|\bx\|_2^2 - \frac{2}{m}\sum_{i=1}^m \hat{y}_i \left\langle S_q(\ba_i), \bx \right\rangle$, where $\hat{y}_i := \sign(y_i)\cdot |y_i|\wedge \tau$ for some thresholding parameter $\tau$. We recall that GMW stands for Gaussian Mean Width ({\em cf.}, Appendix~\ref{app:relateGMW}) and LEP stands for Local Embedding Property ({\em cf.}, Definition \ref{def:lep}). Interested readers may refer to \cite[Table 1]{xu2020quantized} for a summary of further relevant results. 

\begin{sidewaystable}
    \centering
    \caption{Summary of existing results and their associated conditions on the structured set, the observed signal, the sensing model, and the reconstruction algorithm.
    }\label{table:first}
     \begin{tabular}{|c|c|c|c|c|c|} 
    \hline
     & \cite{plan2016generalized} 			& \cite{plan2017high}  & 	\cite{genzel2016high}	& \cite{wei2019statistical} & ({\bf this work})\\ \hline
    \begin{tabular}{@{}c@{}} Signal set \\ $\calK \subseteq \bbR^n$\end{tabular}         & \begin{tabular}{@{}c@{}} Convex; considers \\ (local) GMW on the \\ tangent cone at $\bx^*$\end{tabular}  	&  \begin{tabular}{@{}c@{}} Star-shaped, \\ closed; considers \\(local) GMW\end{tabular}	& \begin{tabular}{@{}c@{}} Convex; considers \\ (local) GMW \end{tabular} 	& \begin{tabular}{@{}c@{}}$G(B_2^k(r))$ \end{tabular} & $G(B_2^k(r))$\\ \hline
    
    \begin{tabular}{@{}c@{}} Condition \\ on $\bx^*$  \end{tabular}                     & $\bx^* \in \calS^{n-1} \cap \frac{1}{\mu}\calK $     & $\mu \frac{\bx^*}{\|\bx^*\|_2} \in \calK$    & $\bx^* \in \calS^{n-1} \cap \frac{1}{\mu}\calK $    & \begin{tabular}{@{}c@{}} $\bx^* \in \calS^{n-1} \cap \calK $, \\ $\lambda \bx^* \in \calK$, \\ $\lambda:=\bbE\left[f'\left(\ba^T \bx^*,\xi\right)\right]$ \end{tabular}
     & $\bx^* \in \calS^{n-1} \cap \frac{1}{\mu}\calK $ \\ \hline
    
    \begin{tabular}{@{}c@{}} Sensing \\ model  \end{tabular}                   & $f(\ba_i^T \bx^*)$	& \begin{tabular}{@{}c@{}} $f(\ba_i^T \bx^*)$ \\ $y_i$ sub-Gaussian \end{tabular} 	& \begin{tabular}{@{}c@{}} $f(\ba_i^T \bx^*)$, with \\ extra adversarial noise \end{tabular}   	& \begin{tabular}{@{}c@{}}$f(\ba_i^T \bx^*, \xi_i)$, \\ $\xi_i$ random noise, \\ $f$ differentiable \\ and deterministic \end{tabular} & \begin{tabular}{@{}c@{}} $f(\ba_i^T \bx^*)$ \\ $y_i$ sub-Gaussian, with \\ extra adversarial noise \end{tabular}\\ \hline
    
    
    Algorithm                                 & $\calK$-Lasso    & PBP    & \begin{tabular}{@{}c@{}} General convex \\ loss functions \end{tabular}    &  Thresholded ERM & $\calK$-Lasso\\ \hline
    
    \begin{tabular}{@{}c@{}}Uniform/ \\Non-uniform \\ guarantee  \end{tabular}            & Non-uniform    & Non-uniform    & Non-uniform    & Non-uniform & \begin{tabular}{@{}c@{}} Non-uniform (and \\ uniform if LEP holds) \end{tabular} \\ \hline
    
    \begin{tabular}{@{}c@{}}Error \\bound   \end{tabular}                           & \begin{tabular}{@{}c@{}} $\mu \bx^* \in \partial\calK$: the \\ dependence on $m$ \\can  be $m^{-\frac{1}{2}}$ \\ (In general, $m^{-\frac{1}{4}}$) \end{tabular}    & \begin{tabular}{@{}c@{}} $\calK = \Sigma_{k}^n$, noiseless\\ 1-bit observations:\\ $\sqrt{\frac{k\log\frac{n}{k}}{m}}$ \\ (In general, $m^{-\frac{1}{4}}$) \end{tabular}    & \begin{tabular}{@{}c@{}} $\calK = \mu \left(\sqrt{k}B_1^n \cap B_2^n\right)$\\ and $\mu \bx^* \in \partial \calK$: $\sqrt{\frac{k\log\frac{n}{k}}{m}}$ \\ (In general, $m^{-\frac{1}{4}}$)\end{tabular}
        & $\sqrt{\frac{k\log Lr}{m}}$ & $\sqrt{\frac{k\log Lr}{m}}$ \\ \hline
    
    \begin{tabular}{@{}c@{}}Measurement\\ vector  $\ba_i$  \end{tabular}            & $\calN(\mathbf{0},\bSigma)$    & $\calN(\mathbf{0},\bI_n)$    & $\calN(\mathbf{0},\bSigma)$    & $S_q(\ba)$ sub-Gaussian & $\calN(\mathbf{0},\bSigma)$ \\ \hline
    \end{tabular} 
\end{sidewaystable} 

\vspace*{-1ex}
\section{Omitted Details and Additional Auxiliary Results for Proving Theorem \ref{thm:generative} (Non-Uniform Recovery)}\label{sec:preliminary}
\vspace*{-1ex}

In this section, we fill in the missing details for proving Theorem \ref{thm:generative}, including a statement of the concentration bound used to establish Lemma \ref{lem:boraSREC_gen}, and a proof for Lemma~\ref{lem:bxStarbxHat}.  We first provide some useful additional auxiliary results that are general, and then some that are specific to our setup.

\subsection{General Auxiliary Results}\label{app:general_aux}

We have the following basic concentration inequality, which is used in the proof of Lemma \ref{lem:boraSREC_gen}.

\begin{lemma}\label{lem:norm_pres}
    {\em (\hspace{1sp}\cite[Lemma~1.3]{vempala2005random})}
     Fix fixed $\bx \in \bbR^n$, we have for any $\epsilon \in (0,1)$ that
     \begin{align} 
          \bbP\left((1-\epsilon)\|\bx\|_2^2 \le \Big\|\frac{1}{\sqrt{m}}\bA\bx\Big\|_2^2\le (1+\epsilon)\|\bx\|_2^2\right)  \ge 1-2e^{-\epsilon^2 (1 - \epsilon) m/4}.
     \end{align}
\end{lemma}

The following definition formally introduces the notion of an $\epsilon$-net, also known as a covering set. 
\begin{definition}\label{def:eps_net}
    Let $(\calX,d)$ be a metric space, and fix $\epsilon>0$. A subset $S \subseteq \calX$ is said be an {\em $\epsilon$-net} of $\calX$ if, for all $\bx \in \calX$, there exists some $\bs \in S$ such that $d(\bx,\bs) \le \epsilon$. The minimal cardinality of an $\epsilon$-net of $\calX$ 
    is denoted by $\calN^*(\calX,\epsilon)$ and is called the covering number of $\calX$ (with parameter $\epsilon$).
\end{definition}

Alongside the sub-Gaussian notion in Definition \ref{def:subg}, we use the following definition of a sub-exponential random variable and sub-exponential norm.

\begin{definition}
 A random variable $X$ is said to be sub-exponential if there exists a positive constant $C$ such that $\left(\bbE\left[|X|^p\right]\right)^{\frac{1}{p}} \le C p$ for all $p \ge 1$. The sub-exponential norm of $X$ is defined as
 \begin{equation}
  \|X\|_{\psi_1} = \sup_{p \ge 1} p^{-1} \left(\bbE\left[|X|^p\right]\right)^{\frac{1}{p}}.
 \end{equation}
\end{definition}
The product of two sub-Gaussian random variables is sub-exponential, as stated in the following.
\begin{lemma}{\em (\hspace{1sp}\cite[Lemma~2.7.7]{vershynin2018high})}\label{lem:prod_subGs}
  Let $X$ and $Y$ be sub-Gaussian random variables (not necessarily independent). Then $XY$ is sub-exponential, and satisfies
  \begin{equation}
   \|XY\|_{\psi_1} \le \|X\|_{\psi_2}\|Y\|_{\psi_2}.
  \end{equation}
\end{lemma}

In our setting, since we assume that $y_i$ is sub-Gaussian and $\langle \ba_i,\bx^* \rangle \sim \calN(0,1)$, Lemma \ref{lem:prod_subGs} reveals that the random variable $y_i \langle \ba_i,\bx^* \rangle$ is sub-exponential, and has the same distribution as $f(g)g$ with $g \sim \calN(0,1)$, yielding
\begin{equation}\label{eq:mupsi}
 \mu = \bbE[f(g)g] \le \bbE[|f(g)g|] \le \|f(g)g\|_{\psi_1} \le C\psi
\end{equation}
for some absolute constant $C > 0$.  In addition, we have the following concentration inequality for sums of independent sub-exponential random variables.
\begin{lemma}{\em (\hspace{1sp}\cite[Proposition~5.16]{vershynin2010introduction})}\label{lem:large_dev}
Let $X_{1}, \ldots , X_{N}$ be independent centered sub-exponential random variables, and $K = \max_{i} \|X_{i} \|_{\psi_{1}}$. Then for every $\balpha = [\alpha_1,\ldots,\alpha_N]^T \in \bbR^N$ and $\epsilon \geq 0$, it holds that
\begin{equation}
 \mathbb{P}\bigg( \Big|\sum_{i=1}^{N}\alpha_i X_{i}\Big|\ge \epsilon\bigg)  \leq 2  \exp \left(-c \cdot \mathrm{min}\Big(\frac{\epsilon^{2}}{K^{2}\|\balpha\|_2^2},\frac{\epsilon}{K\|\balpha\|_\infty}\Big)\right). \label{eq:subexp}
\end{equation}
\end{lemma}

\subsection{Auxiliary Results for Our Setup} \label{app:aux_our_lemmas}

In the remainder of this appendix, we consider the setup described in Section \ref{sec:prob_setup_examples}.
Based on Lemma~\ref{lem:large_dev}, we have the following.
\begin{lemma}\label{lem:useful_dev}
 Fix any $\bar{\bx} \in \calS^{n-1}$ and let $\bar{\by} := f(\bA \bar{\bx})$. For any $t > 0$, if $m = \Omega\left(t +\log n\right)$, then with probability $1-e^{-\Omega(t)}$, we have 
 \begin{equation}
  \left\|\frac{1}{m}\bA^T (\bar{\by} - \mu \bA \bar{\bx})\right\|_\infty \le O\left(\psi\sqrt{\frac{t+\log n}{m}}\right).
 \end{equation}
\end{lemma}
\begin{proof}
 For any fixed $j \in [n]$, let $X_j$ be the $j$-th entry of $\frac{1}{m}\bA^T (\bar{\by} - \mu \bA \bar{\bx})$. We have 
 \begin{align}
  X_j = \frac{1}{m}\sum_{i=1}^m a_{ij} (\bar{y}_i - \mu \langle \ba_i, \bar{\bx}\rangle) = \frac{1}{m}\sum_{i =1}^m X_{ij},
 \end{align}
where $X_{ij} := a_{ij} (\bar{y}_i - \mu \langle \ba_i, \bar{\bx}\rangle)$.  We proceed by showing that $\{X_{ij}\}_{i \in [m]}$ are i.i.d.~sub-exponential random variables. 

Since $\ba_i \sim \calN(\mathbf{0},\bI_n)$, we have $\mathrm{Cov}[a_{ij}, \langle \ba_i, \bar{\bx}\rangle] = \bar{x}_j$. For $i \in [m]$, letting $g := \langle \ba_i, \bar{\bx} \rangle \sim \calN(0,1)$, we find that $a_{ij} \sim \calN(0,1)$ can be written as $a_{ij} = \bar{x}_j g + \sqrt{1-\bar{x}_j^2} h$, where $h \sim \calN(0,1)$ is independent of $g$. Thus, $X_{ij} = a_{ij} (\bar{y}_i - \mu \langle \ba_i, \bar{\bx}\rangle)  = (\bar{x}_j g + \sqrt{1-\bar{x}_j^2} h)(f(g)-\mu g)$, and hence $\bbE[X_{ij}] = \bar{x}_j \bbE[f(g)g - \mu g^2] = \mu - \mu = 0$. In addition, from Lemma~\ref{lem:prod_subGs} and~\eqref{eq:mupsi}, we obtain
\begin{equation}
 \|X_{ij}\|_{\psi_1} \le C' \|f(g) - \mu g\|_{\psi_2} \le C'' \psi.
\end{equation}
For fixed $c' > 0$, letting $\epsilon_j = c' \|X_{1j}\|_{\psi_1} \sqrt{\frac{t + \log n}{m}}$ and $\epsilon = \max_{j} \epsilon_j$, we have from Lemma~\ref{lem:large_dev} that
\begin{align}
 \bbP(|X_j| \ge \epsilon) & \le \bbP(|X_j| \ge \epsilon_j) \\
    &= \bbP \left(\frac{1}{m}\left|\sum_{i=1}^m X_{ij}\right| \ge \epsilon_j \right) \\
    & \le 2\exp\left(-c \min\left(\frac{m\epsilon^2_j}{\|X_{1j}\|_{\psi_1}^2}, \frac{m\epsilon_j}{\|X_{1j}\|_{\psi_1}}\right)\right) \\
    & \le \exp\left(-\Omega(t +\log n)\right), \label{eq:Xj_tail4}
\end{align}
where \eqref{eq:Xj_tail4} uses $m = \Omega\left(t +\log n\right)$ and the choice of $\epsilon_j$.  For sufficiently large $c'$, we can make the implied constant to $\Omega(\cdot)$ in \eqref{eq:Xj_tail4} greater than one, and taking the union bound over $j \in [n]$ gives
\begin{equation}
 \bbP\left(\left\|\frac{1}{m}\bA^T (\bar{\by} - \mu \bA \bar{\bx})\right\|_\infty \ge \epsilon \right) \le n \exp\left(-\Omega(t +\log n)\right) = e^{-\Omega(t)}
\end{equation}
as desired.
\end{proof}

In addition, we have the following useful lemma.
\begin{lemma}\label{lem:varU}
 Fix any $\bar{\bx} \in \calS^{n-1}$ and let $\bar{\by} := f(\bA \bar{\bx})$. For any fixed $\bu \in \bbR^n$, the random variable $U := \frac{1}{m} \left\langle \bu, \bA^T (\bar{\by} - \mu \bA\bar{\bx}) \right\rangle$ has zero mean and is sub-exponential. 
 Moreover, for any $\xi > 0$, if $m = \Omega(\xi^2)$, then with probability $1-e^{-\Omega(\xi^2)}$, we have 
 \begin{equation}
  |U| \le \frac{\xi \psi \|\bu\|_2}{\sqrt{m}}.
 \end{equation}
\end{lemma}
\begin{proof}
 When $\bu$ is the zero vector, the result is trivial, so we only consider $\bu \ne \mathbf{0}$. Following similar steps to the proof of Lemma \ref{lem:useful_dev}, we write
 \begin{align}
  \left\langle \bu, \bA^T (\bar{\by} - \mu \bA\bar{\bx}) \right\rangle & =  \sum_{j=1}^n u_j \sum_{i=1}^m a_{ij} (\bar{y}_i - \mu \langle \ba_i, \bar{\bx}\rangle) \\
  & = \sum_{i=1}^m(\bar{y}_i - \mu \langle \ba_i, \bar{\bx}\rangle) \sum_{j =1}^n u_j a_{ij} \\
  & = \|\bu\|_2 \sum_{i=1}^m(\bar{y}_i - \mu \langle \ba_i, \bar{\bx}\rangle) \langle \ba_i, \bar{\bu}\rangle \\
  & = \|\bu\|_2 \sum_{i=1}^m U_i,
 \end{align}
where $\bar{\bu} = \frac{\bu}{\|\bu\|_2}$ and $U_i := (\bar{y}_i - \mu \langle \ba_i, \bar{\bx}\rangle) \langle \ba_i, \bar{\bu}\rangle$. We proceed by showing that $U_1,\ldots,U_m$ are i.i.d.~sub-exponential random variables. Note that $\langle \ba_i, \bar{\bu}\rangle \sim \calN(0,1)$, and $\mathrm{Cov}[\langle \ba_i, \bar{\bu}\rangle, \langle \ba_i, \bar{\bx}\rangle] = \langle \bar{\bx}, \bar{\bu}\rangle$. 
Fixing $i \in [m]$ and letting $g := \langle \ba_i, \bar{\bx}\rangle \sim \calN(0,1)$, we find that $\langle \ba_i, \bar{\bu}\rangle$ can be written as $\langle \ba_i, \bar{\bu}\rangle = \langle \bar{\bx}, \bar{\bu}\rangle g + \sqrt{1-\langle \bar{\bx}, \bar{\bu}\rangle^2} h$, where $h \sim \calN(0,1)$ is independent of $g$. Therefore, we obtain
\begin{equation}
 \bbE[U_i] = \bbE \left[(\bar{y}_i - \mu \langle \ba_i, \bar{\bx}\rangle) \langle \ba_i, \bar{\bu}\rangle\right] = \langle \bar{\bx}, \bar{\bu}\rangle\bbE [f(g)g-\mu g^2] = 0.
\end{equation}
In addition, from Lemma~\ref{lem:prod_subGs} and~\eqref{eq:mupsi}, we derive 
\begin{equation}
 \|U_i\|_{\psi_1} 
    \le C' \|f(g) - \mu g\|_{\psi_2} \le C'' \psi.
\end{equation}
Letting $\epsilon = c' \frac{\xi \psi \|\bu\|_2}{\sqrt{m}}$, we deduce from Lemma \ref{lem:large_dev} that
\begin{align}
 \bbP(|U| \ge \epsilon) &= \bbP\left(\frac{\|\bu\|_2}{m}\left|\sum_{i=1}^m U_i\right| \ge \epsilon\right) \\
 & \le 2 \exp\left(-c \min \left(\frac{m\epsilon^2}{\|U_i\|_{\psi_1}^2 \|\bu\|_2^2}, \frac{m\epsilon}{\|U_i\|_{\psi_1}\|\bu\|_2}\right)\right) \\
 & \le e^{-\Omega\left(\xi^2\right)}, \label{eq:Utail_final}
\end{align}
where \eqref{eq:Utail_final} follows from $m = \Omega(\xi^2)$ and the choice of $\epsilon$.
\end{proof}

Based on the above results, we are now in a positive to prove Lemma~\ref{lem:bxStarbxHat}.  

\subsection{Proof of Lemma~\ref{lem:bxStarbxHat} (Main Auxiliary Result for Proving Theorem \ref{thm:generative})}

We utilize ideas from \cite{bora2017compressed} based on forming a chain of nets.  Specifically, for a positive integer $l$, let $M = M_0 \subseteq M_1 \subseteq \ldots \subseteq M_l$ be a chain of nets of $B_2^k(r)$ such that $M_i$ is a $\frac{\delta_i}{L}$-net with $\delta_i = \frac{\delta}{2^i}$.  There exists such a chain of nets with \cite[Lemma~5.2]{vershynin2010introduction}
    \begin{equation}
        \log |M_i| \le k \log\frac{4Lr}{\delta_i}. \label{eq:net_size}
    \end{equation}
    By the $L$-Lipschitz assumption on $G$, we have for any $i \in [l]$ that $G(M_i)$ is a $\delta_i$-net of $G(B_2^k(r))$. 
We write $\tilde{\bx}$ as 
\begin{equation}
 \tilde{\bx}  = (\tilde{\bx} - \tilde{\bx}_l) + (\tilde{\bx}_l - \tilde{\bx}_{l-1}) + \ldots + (\tilde{\bx}_1 - \tilde{\bx}_0) + \tilde{\bx}_0, \label{eq:net_decomp}
\end{equation}
where $\tilde{\bx}_i\in G(M_i)$ for all $i \in [l]$, and $\|\tilde{\bx}- \tilde{\bx}_l\|_2 \le \frac{\delta}{2^l}$, $\|\tilde{\bx}_i - \tilde{\bx}_{i-1}\|_2 \le \frac{\delta}{2^{i-1}}$ for all $i \in [l]$. Therefore, the triangle inequality gives
    \begin{equation}\label{eq:hatbx}
        \|\tilde{\bx}-\tilde{\bx}_0\|_2 < 2\delta.
    \end{equation}
We decompose $\frac{1}{m}\langle \bA^T (\bar{\by} -\mu\bA\bar{\bx}),\tilde{\bx}- \mu\bar{\bx} \rangle$ into three terms: 
\begin{align}
 \left\langle \frac{1}{m} \bA^T (\bar{\by} -\mu\bA\bar{\bx}), \tilde{\bx} - \mu\bar{\bx} \right\rangle
 = &\left\langle \frac{1}{m} \bA^T (\bar{\by} -\mu\bA\bar{\bx}), \tilde{\bx}_0 - \mu\bar{\bx} \right\rangle  \nn \\ &+ \sum_{i=1}^l \left\langle \frac{1}{m} \bA^T (\bar{\by} -\mu\bA\bar{\bx}), \tilde{\bx}_i - \tilde{\bx}_{i-1} \right\rangle \nn \\
 & + \left\langle \frac{1}{m} \bA^T (\bar{\by} -\mu\bA\bar{\bx}), \tilde{\bx} - \tilde{\bx}_{l} \right\rangle.\label{eq:three_terms}
\end{align}
We derive upper bounds for these terms separately:
\begin{enumerate}[leftmargin=5ex,itemsep=0ex,topsep=0.25ex]
 \item For any $\bt \in \bbR^n$, from Lemma~\ref{lem:varU}, we have that for any $\xi > 0$, if $m = \Omega\left(\xi^2\right)$, then with probability $1-e^{-\Omega(\xi^2)}$, 
\begin{equation}\label{eq:single_ub}
 \left\langle \frac{1}{m} \bA^T (\bar{\by} -\mu\bA\bar{\bx}), \bt - \mu\bar{\bx} \right\rangle \le \frac{\xi \psi}{\sqrt{m}} \|\bt - \mu\bar{\bx}\|_2.
\end{equation}
Recall that $\log|G(M)| = \log |M| \le k\log\frac{4Lr}{\delta}$. We set $\xi = C \sqrt{k \log \frac{Lr}{\delta}}$ in~\eqref{eq:single_ub}, where $C$ is a certain positive constant, and let $m = \Omega\left(\xi^2\right) = \Omega\left(k \log\frac{Lr}{\delta}\right)$. By the union bound over $G(M)$, we have that with probability $1-e^{-\Omega\left(k\log \frac{Lr}{\delta}\right)}$, for {\em all} $\bt \in G(M)$,
\begin{equation}\label{eq:union_bd}
 \left\langle \frac{1}{m} \bA^T (\bar{\by} -\mu\bA\bar{\bx}), \bt - \mu\bar{\bx} \right\rangle \le O\left(\psi\sqrt{\frac{k\log \frac{Lr}{\delta}}{m}}\right) \|\bt - \mu\bar{\bx}\|_2.
\end{equation}
Therefore, with probability $1-e^{-\Omega\left(k\log \frac{Lr}{\delta}\right)}$, the first term in~\eqref{eq:three_terms} can be upper bounded by 
\begin{align}
 \left\langle \frac{1}{m} \bA^T (\bar{\by} -\mu\bA\bar{\bx}), \tilde{\bx}_0 - \mu\bar{\bx} \right\rangle & \le O\left(\psi\sqrt{\frac{k\log \frac{Lr}{\delta}}{m}}\right) \|\tilde{\bx}_0 - \mu\bar{\bx}\|_2 \\
 & \le O\left(\psi\sqrt{\frac{k\log \frac{Lr}{\delta}}{m}}\right) (\|\tilde{\bx} - \mu\bar{\bx}\|_2 + 2\delta),\label{eq:first_term}
\end{align}
where \eqref{eq:first_term} uses~\eqref{eq:hatbx} and the triangle inequality.

\item From Lemma~\ref{lem:varU}, similarly to~\eqref{eq:single_ub}, 
and applying the union bound, we obtain that for {\em all} $i \in [l]$ with corresponding $\xi_i >0$ and {\em all} $(\bt_{i-1},\bt_{i})$ pairs in $G(M_{i-1}) \times G(M_{i})$, if $m = \Omega\left(\max_i \xi^2_i\right)$, then with probability at least $1-\sum_{i=1}^l |M_{i-1}| \cdot |M_{i}|e^{-\frac{\xi_i^2}{2}}$,
\begin{equation}
 \left\langle \frac{1}{m} \bA^T (\bar{\by} -\mu\bA\bar{\bx}), \bt_i - \bt_{i-1} \right\rangle \le \frac{\xi_i \psi}{\sqrt{m}} \|\bt_i - \bt_{i-1}\|_2.
\end{equation}
Since~\eqref{eq:net_size} gives $\log \left( |M_i| \cdot |M_{i-1}|\right) \le 2i k + 2k\log\frac{4Lr}{\delta}$, if we set $\xi_i = C' \sqrt{ik+ k \log \frac{Lr}{\delta}}$ with $C'$ sufficiently large, we obtain
\begin{equation}
 \sum_{i=1}^l |M_{i-1}| \cdot |M_{i}|e^{-\frac{\xi_i^2}{2}} = \sum_{i=1}^l e^{-\Omega\left(ik + k \log\frac{Lr}{\delta}\right)} = e^{-\Omega\left(k \log\frac{Lr}{\delta}\right)}\sum_{i=1}^l e^{-\Omega(ik)} = e^{-\Omega\left(k \log\frac{Lr}{\delta}\right)}.
\end{equation}
Recall that $\|\tilde{\bx}_i - \tilde{\bx}_{i-1}\|_2 \le \frac{\delta}{2^{i-1}}$ for all $i \in [l]$. Then, we obtain that if $m = \Omega\left(k\left(l + \log\frac{Lr}{\delta}\right)\right)$, with probability $1-e^{-\Omega\left(k\log \frac{Lr}{\delta}\right)}$, the second term in~\eqref{eq:three_terms} can be upper bounded by 
\begin{align}
 \sum_{i=1}^l \left\langle \frac{1}{m} \bA^T (\bar{\by} -\mu\bA\bar{\bx}), \tilde{\bx}_i - \tilde{\bx}_{i-1} \right\rangle &\le \frac{\psi}{\sqrt{m}} \sum_{i=1}^l \xi_i \|\tilde{\bx}_i - \tilde{\bx}_{i-1}\|_2  \\
 & \le C'\psi \sum_{i=1}^l \sqrt{\frac{ik + k\log \frac{Lr}{\delta}}{m}} \times \frac{\delta}{2^{i-1}} \label{eq:second_term0} \\
 & \le C' \psi\delta \sqrt{\frac{k}{m}} \sum_{i=1}^l \frac{\sqrt{i} + \sqrt{\log \frac{Lr}{\delta}}}{2^{i-1}} \label{eq:second_term1} \\
 & = O\left(\psi\delta\sqrt{\frac{k \log \frac{Lr}{\delta}}{m}}\right),\label{eq:second_term}
\end{align}
where \eqref{eq:second_term0} substitutes the choice of $\xi_i$, \eqref{eq:second_term1} uses $\sqrt{a+b} \le \sqrt{a}+\sqrt{b}$, and \eqref{eq:second_term} uses the assumption $Lr = \Omega(\delta n)$ and the fact that $\sum_{i=1}^{\infty} \frac{\sqrt i}{2^{i-1}}$ is finite.

\item With  $m = \Omega\left(k \log \frac{Lr}{\delta}\right)$, if we set $t = \Omega(k\log \frac{Lr}{\delta})$ in Lemma~\ref{lem:useful_dev}, we obtain with probability $1-e^{-\Omega\left(k\log \frac{Lr}{\delta}\right)}$ that
\begin{equation}
 \left\|\frac{1}{m} \bA^T (\bar{\by} -\mu\bA\bar{\bx})\right\|_\infty \le O\left(\psi\sqrt{\frac{k\log \frac{Lr}{\delta}}{m}}\right).
\end{equation}
Then, setting $l = \lceil \log_2 n \rceil$, with probability $1-e^{-\Omega\left(k\log \frac{Lr}{\delta}\right)}$, the third term in~\eqref{eq:three_terms} can be upper bounded as follows: 
\begin{align}
 \left\langle \frac{1}{m} \bA^T (\bar{\by} -\mu\bA\bar{\bx}), \tilde{\bx} - \tilde{\bx}_{l} \right\rangle 
    &\le \left\|\frac{1}{m} \bA^T (\bar{\by} -\mu\bA\bar{\bx})\right\|_\infty \|\tilde{\bx} - \tilde{\bx}_{l}\|_1 \label{eq:third_term0} \\
 & \le O\left(\psi\sqrt{\frac{k\log \frac{Lr}{\delta}}{m}}\right) \sqrt{n} \|\tilde{\bx} - \tilde{\bx}_{l}\|_2 \label{eq:third_term1} \\
 & \le O\left(\psi\sqrt{\frac{k\log \frac{Lr}{\delta}}{m}}\right) \sqrt{n} \times \frac{\delta}{2^l} \label{eq:third_term2} \\
 & = O\left(\psi\delta\sqrt{\frac{k\log \frac{Lr}{\delta}}{m}}\right),\label{eq:third_term}
\end{align}
where \eqref{eq:third_term0} uses H\"older's inequality, \eqref{eq:third_term1} uses $\|\bv\|_1 \le \sqrt{n}\|\bv\|_2$ for $\bv \in \bbR^n$, \eqref{eq:third_term2} uses the definition of $\tilde{\bx}_{l}$, and \eqref{eq:third_term} uses $l = \lceil \log_2 n \rceil$.
\end{enumerate}
By the assumption $Lr = \Omega(\delta n)$, the choice $l = \lceil \log_2 n \rceil$ leads to $m = \Omega\left(k\left(l + \log\frac{Lr}{\delta}\right)\right) = \Omega\left(k\log\frac{Lr}{\delta}\right)$. Substituting \eqref{eq:first_term},~\eqref{eq:second_term}, and \eqref{eq:third_term} into \eqref{eq:three_terms}, we obtain that when $m = \Omega\left(k\log\frac{Lr}{\delta}\right)$, with probability $1-e^{-\Omega\left(k\log \frac{Lr}{\delta}\right)}$,
\begin{equation}
 \left\langle \frac{1}{m}\bA^T(\bar{\by} - \mu\bA \bar{\bx}), \tilde{\bx}-\mu \bar{\bx} \right \rangle \le O\left(\psi\sqrt{\frac{k\log \frac{Lr}{\delta}}{m}}\right)\|\tilde{\bx} - \mu\bar{\bx}\|_2 + O\left(\delta\psi\sqrt{\frac{k\log \frac{Lr}{\delta}}{m}}\right).
\end{equation} 
This completes the proof of Lemma~\ref{lem:bxStarbxHat}.

\vspace*{-1ex}
\section{Omitted Proofs from Section \ref{sec:extensions} (Other Extensions)} \label{sec:pf_other}
\vspace*{-1ex}

\subsection{Proof Outline for Corollary \ref{coro:bddSparse} (Bounded Sparse Vectors)} \label{sec:pf_sparse}
\vspace*{-1ex}

For fixed $\nu >0$, let $\calS_\nu := \Sigma_{k}^n \cap \nu B_2^n$,  where $\Sigma_k^n$ represents the set of $k$-sparse vectors in $\bbR^n$. We know that for any $\delta>0$, there exists a $\delta$-net $\calM_\nu$ of $\calS_\nu$ with $|\calM_\nu| \le \binom{n}{k} \left(\frac{\nu}{\delta}\right)^k \le \left(\frac{e n \nu}{k \delta}\right)^k = \exp\left(O\left(k \log \frac{\nu n}{\delta k}\right)\right)$~\cite{baraniuk2008simple}. 
Using this observation and following the proof of Theorem \ref{thm:generative}, we can derive the Corollary~\ref{coro:bddSparse} for the case that the signal comes from the set of bounded $k$-sparse vectors.

\vspace*{-1ex}
\subsection{Proof of Corollary \ref{coro:arb_cov} (General Covariance Matrices)} \label{sec:pf_gen_cov}
\vspace*{-1ex}

We can write $\ba_i$ as $\ba_i = \sqrt{\bSigma}\bb_i$ with $\bb_i \sim \calN(\mathbf{0},\bI_n)$. Letting\footnote{For matrices $\bV_{1} \in \mathbb{R}^{F_{1}\times N}$ and $\bV_{2} \in \mathbb{R}^{F_{2}\times N}$, we let $\left[\bV_{1}; \bV_{2}\right]$ denote the vertical concatenation.} $\bA = \left[\ba_1^T;\ba_2^T;\ldots;\ba_m^T\right] \in \bbR^{m \times n}$ and $\bB = \left[\bb_1^T;\ldots;\bb_m^T\right] \in \bbR^{m \times n}$, we have
\begin{align}
 & \hat{\bx} = \arg \min_{\bx \in \calK} \|\by - \bA \bx\|_2 \nonumber \\
 \Leftrightarrow~  &  \hat{\bx} = \arg \min_{\bx \in \calK} \|\by - \bB \sqrt{\bSigma} \bx\|_2 \\
\Leftrightarrow~ &  \sqrt{\bSigma}\hat{\bx} = \arg \min_{\bx \in \sqrt{\bSigma}\calK} \|\by - \bB \bx\|_2. 
\end{align}
Define $\hat{G}$ as $ \hat{G}(\bz) = \sqrt{\bSigma} G(\bz)$ for all $\bz \in B_2^k(r)$. Then, it is straightforward to establish that $\hat{G}$ is $\hat{L}$-Lipschitz with $\hat{L}= \|\bSigma\|_{2\to 2}^{\frac{1}{2}} L$. In addition, we have $\by = f(\bA \bx^*) = f(\bB \sqrt{\bSigma} \bx^*)$, $\|\sqrt{\bSigma} \bx^*\|_2 =1$ and $\mu \big(\sqrt{\bSigma} \bx^*\big) \in \sqrt{\bSigma}\calK = \hat{G}(B_2^k(r))$. Applying Theorem~\ref{thm:generative}, we obtain that when $\|\bSigma\|_{2\to 2}^{\frac{1}{2}} Lr = \Omega(\epsilon \psi n)$ and $m = \Omega\big(\frac{k}{\epsilon^2} \log \frac{\|\bSigma\|_{2\to 2}^{\frac{1}{2}} Lr}{\epsilon \psi}\big)$, with probability $1-e^{-\Omega\left(\epsilon^2 m\right)}$,
 \begin{equation}
  \|\sqrt{\bSigma}\hat{\bx}-\mu \sqrt{\bSigma} \bx^*\|_2 \le \psi \epsilon + \tau,
 \end{equation}
 as desired. 

\vspace*{-1ex}
\section{Alternative Model for Binary Measurements} \label{app:alterBinary}
\vspace*{-1ex}

For binary observations, the following measurement model is considered in various works \cite{plan2012robust,zhang2014efficient,zhu2015towards,chen2015one}:  The response variables, $y_i \in \{-1,1\}, i \in [m]$, are drawn independently at random according to some distribution satisfying
\begin{equation}\label{eq:binary_gen}
 \bbE [y_i|\ba_i] = \theta(\ba_i^T \bx^*),
\end{equation}
for some deterministic function $\theta$ with $-1 \le  \theta (z) \le 1$.  In this section, we provide a result related to Theorem \ref{thm:generative} for this model, again considering the case that $\ba_i \sim \calN(\mathbf{0},\bI_n)$ and $\bx^* \in \calK \cap \calS^{n-1}$  with $\calK=G(B_2^k(r))$ for some $L$-Lipschitz generative model $G$.

The model \eqref{eq:binary_gen} is a special case of \eqref{eq:yi} in which $f(g) \in \{-1,1\}$ and $\bbE[f(g)] = \theta(g)$.  Using this interpretation and the tower property of expectation, we readily find that
\begin{equation}
    \mu = \bbE[ \bbE[ f(g) g \,|\, g ]] = \bbE [\theta(g) g]
\end{equation}
with $g \sim \calN(0,1)$. In addition, we have for any $i \in [m]$ that
\begin{equation}
 \bbE [y_i \ba_i^T \bx^*] = \bbE [\bbE[y_i \ba_i^T \bx^* | \ba_i] ] = \bbE [(\ba_i^T \bx^*) \theta(\ba_i^T \bx^*)] = \mu,
\end{equation}
and it is straightforward to show that~\cite[Lemma~4]{zhang2014efficient}
\begin{equation}\label{eq:useful_exp}
 \bbE [y_i \ba_i ] = \mu \bx^*.
\end{equation}
Let $\tilde{\by} \in \{-1,1\}^m$ be a vector of corrupted observations satisfying $\frac{1}{\sqrt{m}}\|\by - \tilde{\by}\|_2 \le \tau$. 
To derive an estimator for $\bx^*$, we seek $\hat{\bx}$ {\em maximizing} $\tilde{\by}^T(\bA \bx)$ over $\bx \in \calK=G(B_2^k(r))$, i.e.,
\begin{equation}\label{eq:corr_max}
 \hat{\bx}:=\arg \max_{\bx \in \calK} \tilde{\by}^T(\bA \bx).
\end{equation}
As was done in previous works such as \cite{plan2012robust,zhang2014efficient}, we assume that the considered low-dimensional set is contained in the unit Euclidean ball, i.e., $\calK \subseteq B_2^n$.  In this section, we establish the following theorem, which is similar to Theorem~\ref{thm:generative}. Although the ideas are similar, the model assumptions and the algorithms used are slightly different, so the results are both of interest.


\begin{theorem}\label{thm:generative_binary}
Consider any $\bx^* \in \calK \cap \calS^{n-1}$ with $\calK = G(B_2^k(r)) \subseteq B_2^n$ for some $L$-Lipschitz generative model $G \,:\, B_2^k(r) \to \bbR^n$, along with $\by$ generated from the model \eqref{eq:binary_gen} with $\ba_i \overset{i.i.d.}{\sim} \calN(\mathbf{0},\bI_n)$, and an arbitrary corrupted vector $\tilde{\by}$ with $\frac{1}{\sqrt{m}} \|\tilde{\by}-\by\|_2 \le \tau$.
 For any $\epsilon >0$, if $Lr = \Omega(\epsilon n)$ and $m = \Omega\left(\frac{k}{\epsilon^2}\log \frac{Lr}{\epsilon}\right)$, then with probability $1-e^{-\Omega(\epsilon^2 m)}$, any solution $\hat{\bx}$ to~\eqref{eq:corr_max} satisfies 
 \begin{equation}
  \|\bx^* - \hat{\bx}\|_2 \le \frac{\epsilon + \tau}{\mu}.
 \end{equation}
\end{theorem}

The proof is mostly similar to that of Theorem \ref{thm:generative}, so we only outline the differences in the following.

\vspace*{-1ex}
\subsection{Auxiliary Results}\label{sec:thm_binary}
\vspace*{-1ex}

In the remainder of this appendix, we assume that the binary vector $\by$ is generated according to~\eqref{eq:binary_gen}.
Note that for binary measurements, the relevant random variables are sub-Gaussian, and thus we only need concentration inequalities for sub-Gaussian random variables, instead of those for sub-exponential random variables. According to~\cite[Proposition~5.10]{vershynin2010introduction}, we have the following concentration inequality for sub-Gaussian random variables. 
\begin{lemma} {\em (Hoeffding-type inequality \cite[Proposition~5.10]{vershynin2010introduction})}
\label{lem:large_dev_Gaussian} Let $X_{1}, \ldots , X_{N}$ be independent zero-mean sub-Gaussian random variables, and let $K = \max_i \|X_i\|_{\psi_2}$. Then, for any $\balpha=[\alpha_1,\alpha_2,\ldots,\alpha_N]^T \in \mathbb{R}^N$ and any $t\ge 0$, it holds that
\begin{equation}
\mathbb{P}\left( \Big|\sum_{i=1}^{N} \alpha_i X_{i}\Big| \ge t\right) \le   \exp\left(1-\frac{ct^2}{K^2\|\balpha\|_2^2}\right),
\end{equation}
where $c>0$ is a constant.
\end{lemma}
By Lemma~\ref{lem:large_dev_Gaussian} and the equality $\bbE[y_i \ba_i] = \lambda \bx^*$, we arrive at the following lemma, which is similar to Lemma~\ref{lem:useful_dev}.
\begin{lemma}{\em \hspace{1sp}\cite[Lemma~3]{zhang2014efficient}}\label{lem:useful_dev_Gaussian}
 With probability at least $1-e^{1-t}$, we have
 \begin{equation}
  \left\|\frac{1}{m}\bA^T \by  -\lambda \bx^*\right\|_\infty \le c\sqrt{\frac{t + \log n}{m}}
 \end{equation}
for a certain constant $c >0$.
\end{lemma}

The following lemma is proved similarly to Lemma~\ref{lem:varU}, so the details are omitted.
\begin{lemma}\label{lem:varU_Gaussian}
 For any $\bu \in \bbR^n$, the random variable $U:=\left\langle \frac{1}{m}\bA^T \by - \lambda \bx^*, \bu \right\rangle$ is sub-Gaussian with zero mean. Moreover, for any $\xi >0$, with probability $1-e^{-\Omega(\xi^2)}$, we have
 \begin{equation}
  |U| \le \frac{\xi \|\bu\|_2}{\sqrt{m}}.
 \end{equation}
\end{lemma}
Finally, based on Lemmas~\ref{lem:useful_dev_Gaussian} and~\ref{lem:varU_Gaussian}, and by using a chain of nets similarly to \eqref{eq:net_size}--\eqref{eq:net_decomp}, we derive the following analog of Lemma~\ref{lem:bxStarbxHat}, whose proof is again omitted due to similarity. Note that Lemmas~\ref{lem:useful_dev_Gaussian} and~\ref{lem:varU_Gaussian} are only used to derive Lemma~\ref{lem:bxStarbxHat_1bit}, and they are not directly used in the proof of Theorem~\ref{thm:generative_binary}. 
\begin{lemma}\label{lem:bxStarbxHat_1bit}
 For any $\delta >0$, if $Lr = \Omega(\delta n)$ and $m = \Omega\big(k \log\frac{Lr}{\delta}\big)$, then with probability $1-e^{-\Omega(k \log\frac{Lr}{\delta})}$, it holds that
 \begin{equation}
  \left\langle \frac{1}{m}\bA^T \by - \lambda \bx^*, \hat{\bx}-\bx^* \right\rangle \le O\left(\sqrt{\frac{k \log \frac{Lr}{\delta}}{m}}\right) \|\bx^* -\hat{\bx}\|_2 + O\left(\delta \sqrt{\frac{k \log \frac{Lr}{\delta}}{m}}\right).
 \end{equation}
\end{lemma}

\vspace*{-1ex}
\subsection{Proof Outline for Theorem \ref{thm:generative_binary}} \label{app:pf_binary}
\vspace*{-1ex}

Because  $\hat{\bx}$ maximizes $\tilde{\by}^T(\bA \bx)$ within $\calK$ and we assume $\bx^* \in \calK$, we obtain
 \begin{equation}
  \tilde{\by}^T (\bA \hat{\bx}) \ge  \tilde{\by}^T (\bA \bx^*),
 \end{equation}
which gives the following after some simple manipulations:
\begin{equation}\label{eq:imp_eq}
 \langle \mu \bx^*, \bx^* -\hat{\bx}\rangle \le \left\langle \frac{1}{m} \bA^T \tilde{\by} - \mu \bx^*, \hat{\bx} - \bx^* \right\rangle.
\end{equation}
Using $\|\hat{\bx}\|_2 \le 1$ and $\|\bx^*\|_2 =1$, we derive a lower bound for $\langle \mu \bx^*, \bx^* -\hat{\bx}\rangle$, i.e., 
\begin{equation}\label{eq:1-bitLB}
      \frac{\mu}{2}\|\hat{\bx}-\bx^*\|_2^2 \le \langle \mu \bx^*, \bx^* -\hat{\bx}\rangle.
\end{equation}
Once this result is in place, the analysis proceeds similarly to that of Theorem~\ref{thm:generative}: Similar to~\eqref{eq:main_first_term}, we derive an upper bound for the adversarial noise term, and using Lemma~\ref{lem:bxStarbxHat_1bit} (which is similar to Lemma~\ref{lem:bxStarbxHat}) to derive the following analog of \eqref{eq:beUsefulForUnion}:
\begin{equation}\label{eq:1-bitUB}
     \left\langle \frac{1}{m}\bA^T \tilde{\by}-\mu \bx^*, \hat{\bx}-\bx^*\right\rangle \le \left(\tau + \sqrt{\frac{k \log \frac{Lr}{\delta}}{m}}\right) \|\bx^* - \hat{\bx}\|_2 + O\left(\tau \delta + \delta \sqrt{\frac{k \log \frac{Lr}{\delta}}{m}}\right). 
\end{equation}
Combining~\eqref{eq:1-bitLB} and~\eqref{eq:1-bitUB}, and using similar steps to those following~\eqref{eq:non_unif_final} in the proof of Theorem~\ref{thm:generative}, we derive the desired upper bound for $\|\bx^* - \hat{\bx}\|_2$.   The details are omitted to avoid repetition. 

\vspace*{-1ex}
\section{Relation to the Gaussian Mean Width}\label{app:relateGMW}
\vspace*{-1ex}

The (global) Gaussian mean width (GMW) of a set $\calK$ is defined as
\begin{equation}
 \omega (\calK) := \bbE \left[\sup_{\bx \in \calK - \calK} \langle \bg,\bx\rangle\right], 
\end{equation}
where $\calK-\calK := \{\bs - \bt \,:\, \bs \in \calK, \bt \in \calK\}$ and $\bg \sim \calN(\bzero,\bI_n)$. The GMW of $\calK$ is a geometric parameter, and is useful for understanding the effective dimension of $\calK$ in estimation problems. In various related works such as~\cite{plan2017high,plan2016generalized}, the sample complexity derived depends directly on the GMW or its local variants.  For example, if $\calK \subseteq \bbR^n$ is compact and star shaped, then  by \cite[Eq.~(2.1)]{plan2017high}, $m = O\big(\frac{\omega(\calK)^2}{\epsilon^4}\big)$ measurements suffice for $\epsilon$-accurate recovery.

According to~\cite{plan2012robust}, the GMW satisfies the following properties:
\begin{enumerate}[leftmargin=5ex,itemsep=0ex,topsep=0.25ex]
 \item If $\calK = B_2^n$ or $\calK = \calS^{n-1}$, then $\omega(\calK) = \bbE [\|\bg\|_2] \le \left(\bbE \left[\|\bg\|_2^2\right]\right)^{1/2} = \sqrt{n}$;
 \item If $\calK$ is a finite set contained in $B_2^n$, then $\omega(\calK) \le C\sqrt{\log |\calK|}$.
\end{enumerate}
Using these observations, we obtain the following lemma. 
\begin{lemma}\label{lem:ub_gaussianWidth}
 Fix $r > 0$, and let $G$ be an $L$-Lipschitz generative model with $Lr =\Omega(1)$, and let $\calK = G(B_2^k(r)) \subseteq B_2^n$. Then, we have
 \begin{equation}
  \omega(\calK)^2 = \Theta\left(k \log \frac{Lr\sqrt{n}}{\sqrt{k}}\right).
 \end{equation}
\end{lemma}
\begin{proof}
 As we stated in \eqref{eq:net_size}, for any $\delta > 0$, there exists a set $M \subseteq B_2^k(r)$ being a $\frac{\delta}{L}$-net of $B_2^k(r)$ with $\log |M| \le k\log\frac{4Lr}{\delta}$, and $G(M)$ is a $\delta$-net of $\calK$. For any $\bx \in \calK - \calK$, there exists $\bs \in G(M)-G(M)$ with $\|\bx-\bs\|_2\le 2\delta$; hence,
 \begin{equation}
  \langle \bg,\bx \rangle \le \langle \bg,\bs \rangle + \|\bg\|_2 \|\bx-\bs\|_2 \le \langle \bg,\bs \rangle + 2\delta\|\bg\|_2.
\end{equation}
As a result, we have
\begin{align}
 \omega(\calK)  & = \bbE \left[\sup_{\bx \in \calK -\calK} \langle \bg,\bx \rangle\right] \\
 & \le \omega(G(M)) +2\delta \bbE [\|\bg\|_2] \\
 & \le C\sqrt{ k \log \frac{4Lr}{\delta}} + 2\delta \sqrt{n}.
\end{align}
By a similar argument, we also have
\begin{equation}
 \omega(\calK) \ge C\sqrt{ k \log \frac{4Lr}{\delta}} - 2\delta \sqrt{n}.
\end{equation}
Setting $\delta = \sqrt{\frac{k}{n}}$ and applying the assumption $Lr = \Omega(1)$, we obtain the desired result.
\end{proof}

We emphasize that the above analysis assumes that $G(B_2^k(r)) \subseteq B_2^n$, and in the absence of such an assumption, the Gaussian mean width $\omega(\calK)$ will generally grow linearly with the radius.

Returning to the sample complexity $m = O\big(\frac{k}{\epsilon^2} \log \frac{Lr}{\psi \epsilon} \big)$ in Theorem \ref{thm:generative}, we find that this reduces to $m = O\big(\frac{\omega(\calK)^2}{\epsilon^2}\big)$ in broad scaling regimes.  For instance, this is the case when $\psi$ is constant, $Lr = n^{\Omega(1)}$ (as is typical for neural networks \cite{bora2017compressed}), and $\epsilon$ decays no faster than polynomially in $n$.

\section{Local Embedding Property (LEP) for the 1-bit Model} \label{app:lep}

For $\bv,\bv' \in \bbR^m$, let $\rmd_{\rm H}(\bv,\bv') := \frac{1}{m}\sum_{i=1}^m \boldsymbol{1} \{ v_i \ne v'_i \}$ denote the (normalized) Hamming distance. Note that when $f(x) = \mathrm{sign}(x)$, we obtain $\mu = \bbE[f(g)g]=\sqrt{\frac{2}{\pi}}$ and $\psi =1$. We have the following lemma, which essentially states that for all $\bx,\bs \in \calS^{n-1}$, if $\bx$ is close to $\bs$ in $\ell_2$-norm, then $\mathrm{sign}(\bA\bx)$ is close to $\mathrm{sign}(\bA\bs)$ in Hamming distance. 
\begin{lemma}{\em \hspace{1sp} (Adapted from~\cite[Corollary~2]{liu2020sample})}\label{lem:main_noiseless}
     For fixed $\epsilon \in (0,1)$, if $m = \Omega\big(\frac{k}{\epsilon}\log \frac{L r}{\mu \epsilon^2}\big)$, with probability $1-e^{-\Omega(\epsilon m)}$, for all $\bx_1,\bx_2 \in \calS^{n-1}$ with $\mu_1\bx_1, \mu_2\bx_2 \in \calK$, where $\mu_1,\mu_2 = \Theta(\mu)$, it holds that 
     \begin{equation}\label{eq:unif_1bit}
         \|\bx_1-\bx_2\|_2 \le \epsilon \Rightarrow \rmd_\rmH(\mathrm{sign}(\bA\bx_1),\mathrm{sign}(\bA\bx_2)) \le O(\epsilon).
     \end{equation}
\end{lemma}
Note that each entry of $|\mathrm{sign}(\bA\bx_1) - \mathrm{sign}(\bA\bx_2)|$ is either $2$ or $0$. Hence, if~\eqref{eq:unif_1bit} is satisfied, we have 
\begin{align}
 & \frac{1}{\sqrt{m}} \|\mathrm{sign}(\bA\bx_1) - \mathrm{sign}(\bA\bx_2)\|_2 = 2 \sqrt{\rmd_{\rmH}(\mathrm{sign}(\bA\bx_1),\mathrm{sign}(\bA\bx_2))} \le O(\sqrt{\epsilon}).
\end{align}
That is, setting $\beta = \frac{1}{2}$, we have that $f(x) = \mathrm{sign}(x)$ satisfies Assumption \ref{as:lep} in Section \ref{sec:uniform} with $\funcM(\delta,\beta) =O \big(\frac{k}{\delta}\log \frac{L r}{\mu\delta^2}\big)$ and $\funcP(\delta,\beta) = 1-e^{-\Omega(\delta m)}$.

\vspace*{-1ex}
\section{Proof of Theorem \ref{thm:generative_uniform} (Uniform Recovery)} \label{app:pf_uniform}
\vspace*{-1ex}

We briefly repeat the argument at the start of the proof of Lemma \ref{lem:bxStarbxHat}:  For fixed $\delta \in (0,1)$ and a positive integer $l$, let $M = M_0 \subseteq M_1 \subseteq \ldots \subseteq M_l$ be a chain of nets of $B_2^k(r)$ such that $M_i$ is a $\frac{\delta_i}{L}$-net with $\delta_i = \frac{\delta}{2^i}$.  There exists such a chain of nets with 
    \begin{equation}
        \log |M_i| \le k \log\frac{4Lr}{\delta_i}. \label{eq:net_size2}
    \end{equation}
    By the $L$-Lipschitz assumption on $G$, we have for any $i \in [l]$ that $G(M_i)$ is a $\delta_i$-net of $G(B_2^k(r))$. 
We write $\mu\bx^*$ and $\hat{\bx}$ as 
\begin{align}
\mu\bx^*  &= (\mu\bx^* - \mu\bx^*_l) + (\mu\bx^*_l - \mu\bx^*_{l-1}) + \ldots + (\mu\bx^*_1 - \mu\bx^*_0) + \mu\bx^*_0, \\
 \hat{\bx} &= (\hat{\bx} - \hat{\bx}_l) + (\hat{\bx}_l - \hat{\bx}_{l-1}) + \ldots + (\hat{\bx}_1 - \hat{\bx}_0) + \hat{\bx}_0, 
\end{align}
where $\hat{\bx}_i, \mu\bx^*_i\in G(M_i)$ for all $i \in [l]$, and $\|\hat{\bx}- \hat{\bx}_l\|_2 \le \frac{\delta}{2^l}$, $\|\mu\bx^*- \mu\bx^*_l\|_2 \le \frac{\delta}{2^l}$, and $\|\hat{\bx}_i - \hat{\bx}_{i-1}\|_2 \le \frac{\delta}{2^{i-1}}$, $\|\mu\bx^*_i - \mu\bx^*_{i-1}\|_2 \le \frac{\delta}{2^{i-1}}$ for all $i \in [l]$. Therefore, the triangle inequality gives
    \begin{equation}\label{eq:hatbxbxstar}
        \|\hat{\bx}-\hat{\bx}_0\|_2 < 2\delta, \quad \|\mu\bx^*-\mu\bx^*_0\|_2 < 2\delta.
    \end{equation}
    In analogy with \eqref{eq:three_terms}, we write
    \begin{align}
     & \left\langle \frac{1}{m} \bA^T (\tilde{\by} - \mu\bA \bx^*), \hat{\bx} - \mu\bx^* \right\rangle \nonumber \\
     & = \left\langle \frac{1}{m} \bA^T (\tilde{\by} - \by), \hat{\bx} - \mu\bx^* \right\rangle + \left\langle \frac{1}{m} \bA^T \left(\by - f\left(\bA\frac{\bx^*_0}{\|\bx^*_0\|_2}\right)\right), \hat{\bx} - \mu\bx^* \right\rangle \nonumber \\
     & + \left\langle \frac{1}{m} \bA^T \left(f\left(\bA\frac{\bx^*_0}{\|\bx^*_0\|_2}\right) - \mu\bA\frac{\bx^*_0}{\|\bx^*_0\|_2}\right), \hat{\bx} - \mu\bx^* \right\rangle + \left\langle \frac{1}{m}\bA^T\mu\bA \left(\frac{\bx^*_0}{\|\bx^*_0\|_2} - \bx^*\right), \hat{\bx} - \mu\bx^* \right\rangle \label{eq:four_terms}
    \end{align}
and proceed by deriving uniform upper bounds for the four terms in~\eqref{eq:four_terms} separately. In the following, we assume that $m = \Omega\left(k \log \frac{Lr}{\delta}\right)$; we will later choose $\delta$ such that this reduces to $m = \Omega\left(k \log \frac{Lr}{\epsilon}\right)$, as in the theorem statement.

\begin{enumerate}[leftmargin=5ex,itemsep=0ex,topsep=0.25ex]
 \item \underline{A uniform upper bound for $\big\langle \frac{1}{m} \bA^T (\tilde{\by} - \by), \hat{\bx} - \mu\bx^* \big\rangle$}: Recall that from~\eqref{eq:main_first_term}, we have
 \begin{align}
     \left\langle \frac{1}{m} \bA^T (\by - \tilde{\by}), \hat{\bx} - \mu\bx^* \right\rangle &\le \left\|\frac{1}{\sqrt{m}}(\by - \tilde{\by})\right\|_2 \times \left\|\frac{1}{\sqrt{m}} \bA (\hat{\bx} - \mu\bx^*)\right\|_2 \\ 
     & \le \tau O(\|\hat{\bx} - \mu\bx^*\|_2 + \delta).\label{eq:uniform_first_term}
    \end{align}
    This inequality holds uniformly for all $\hat{\bx}, \mu \bx^* \in \calK$, since it is based on the uniform result in Lemma \ref{lem:boraSREC_gen}.
    
    \item \underline{A uniform upper bound for $\big\langle \frac{1}{m} \bA^T \left(\by - f\left(\bA\frac{\bx^*_0}{\|\bx^*_0\|_2}\right)\right), \hat{\bx} - \mu\bx^* \big\rangle$}: From~\eqref{eq:hatbxbxstar}, we have $\|\bx^* - \bx^*_0\|_2 \le \frac{2\delta}{\mu}$. Because $\|\bx^*\|_2 =1$ and $\|\bx^* - \bx^*_0\|_2 \ge \left|\|\bx^*_0\|_2-\|\bx^*\|_2\right|$, we obtain 
    \begin{equation}
     \left\|\bx^*_0 - \frac{\bx^*_0}{\|\bx^*_0\|_2}\right\|_2 = \left\|\frac{\bx^*_0 (\|\bx^*_0\|_2 -1)}{\|\bx^*_0\|_2}\right\|_2 \le \big| \|\bx_0^*\|_2 - 1 \big| \le \frac{2\delta}{\mu}, \label{eq:norm_bound0}
    \end{equation}
    and the triangle inequality gives
    \begin{equation}
     \left\|\bx^* - \frac{\bx^*_0}{\|\bx^*_0\|_2} \right\|_2\le \frac{4\delta}{\mu}. \label{eq:x*_bound}
    \end{equation}
    If we choose $\delta \le c' \mu$ for sufficiently small $c'$, then we obtain $\|\bx_0^*\|_2 \in [1-\eta_0,1+\eta_0]$ for arbitrarily small $\eta_0$, which implies that $c\frac{\bx^*_0}{\|\bx^*_0\|_2} \in \calK$ for some $c \in [\mu-\eta,\mu+\eta]$ and arbitrarily small $\eta > 0$ (since $\mu\bx^*_0 \in \calK$ and $\mu = \Theta(1)$). 
    Hence, considering Assumption \ref{as:lep}, we observe that the high-probability LEP condition \eqref{eq:lep} therein (along with $\mu = \Theta(1)$) implies
    \begin{equation}
     \frac{1}{\sqrt{m}} \left\|\by - f\left(\bA\frac{\bx^*_0}{\|\bx^*_0\|_2}\right)\right\|_2 = \frac{1}{\sqrt{m}} \left\|f(\bA \bx^*) - f\left(\bA\frac{\bx^*_0}{\|\bx^*_0\|_2}\right)\right\|_2 \le O\left(\delta^\beta\right),\label{eq:footnote3}
    \end{equation}
    Then, similarly to the derivation of~\eqref{eq:uniform_first_term}, we have that if $m \ge \funcM(\delta,\beta) + \Omega\left( k \log\frac{Lr}{\delta}\right)$, then with probability $1-\funcP(\delta,\beta)-e^{-\Omega(m)}$, 
    \begin{align}
     & \left\langle \frac{1}{m} \bA^T \left(\by - f\left(\bA\frac{\bx^*_0}{\|\bx^*_0\|_2}\right)\right), \hat{\bx} - \mu\bx^* \right\rangle \nn \\
     &\quad \le \left\|\frac{1}{\sqrt{m}}\left(\by-f\left(\bA\frac{\bx^*_0}{\|\bx^*_0\|_2}\right)\right)\right\|_2 \times \left\|\frac{1}{\sqrt{m}} \bA (\hat{\bx} - \mu\bx^*)\right\|_2 \\ 
     &\quad \le O(\delta^\beta) \times O(\|\hat{\bx} - \mu \bx^*\|_2 + \delta) \\
     &\quad = O(\delta^\beta\|\hat{\bx} - \mu\bx^*\|_2 +\delta^{\beta +1}).\label{eq:uniform_second_term}
    \end{align}
    
    \item \underline{A uniform upper bound for $\big\langle \frac{1}{m} \bA^T \left(f\left(\bA\frac{\bx^*_0}{\|\bx^*_0\|_2}\right) - \mu\bA \frac{\bx^*_0}{\|\bx^*_0\|_2}\right), \hat{\bx} - \mu\bx^* \big\rangle$}:
    For brevity, let $\bs_0 = \frac{1}{m} \bA^T \left(f\left(\bA\frac{\bx^*_0}{\|\bx^*_0\|_2}\right) - \mu\bA \frac{\bx^*_0}{\|\bx^*_0\|_2}\right)$. We have
    \begin{align}
      \left\langle \bs_0, \hat{\bx} - \mu\bx^* \right\rangle  = \left\langle \bs_0, \hat{\bx} - \mu\frac{\bx^*_0}{\|\bx^*_0\|_2} \right\rangle +  \left\langle \bs_0, \mu\left(\frac{\bx^*_0}{\|\bx^*_0\|_2} - \bx^*\right) \right\rangle.\label{eq:total_three_terms}
    \end{align}
    By Lemma~\ref{lem:bxStarbxHat} and the union bound over $G(M)$ (for $\bx^*_0$), we obtain with probability $1- |M|e^{-\Omega(k\log\frac{Lr}{\delta})} = 1- e^{-\Omega(k\log\frac{Lr}{\delta})}$ that
    \begin{align}
     \left\langle \bs_0, \hat{\bx} - \mu\frac{\bx^*_0}{\|\bx^*_0\|_2} \right\rangle &\le O\left(\sqrt{\frac{k\log\frac{Lr}{\delta}}{m}} \right) \left\|\hat{\bx} - \mu\frac{\bx^*_0}{\|\bx^*_0\|_2}\right\|_2 + O\left(\delta\sqrt{\frac{k\log\frac{Lr}{\delta}}{m}}\right) \\
     & \le O\left(\sqrt{\frac{k\log\frac{Lr}{\delta}}{m}}\right) (\|\hat{\bx} - \mu\bx^*\|_2 + 4\delta) + O\left(\delta\sqrt{\frac{k\log\frac{Lr}{\delta}}{m}}\right) \label{eq:total_three_terms_one0} \\
     & = O\left(\sqrt{\frac{k\log\frac{Lr}{\delta}}{m}}\right) \|\hat{\bx} - \mu\bx^*\|_2 + O\left(\delta\sqrt{\frac{k\log\frac{Lr}{\delta}}{m}}\right), \label{eq:total_three_terms_one}
    \end{align}
    where \eqref{eq:total_three_terms_one0} follows from the triangle inequality and \eqref{eq:x*_bound}.
    In addition, we have 
    \begin{align}
      & \left\langle  \bs_0, \mu\left(\frac{\bx^*_0}{\|\bx^*_0\|_2} - \bx^*\right) \right\rangle \nn \\
      & = \left\langle \bs_0, \mu\left(\frac{\bx^*_0}{\|\bx^*_0\|_2} - \bx^*_0\right) \right\rangle + \left\langle \bs_0, \mu(\bx_l^* - \bx^*) \right\rangle + \sum_{i=1}^l \left\langle \bs_0, \mu(\bx_{i-1}^* - \bx^*_i) \right\rangle.\label{eq:total_three_terms_two}
    \end{align}
    Then, by Lemma~\ref{lem:varU} and the union bound over $G(M)$ (for $\bx^*_0$), we obtain with probability $1- e^{-\Omega(k\log\frac{Lr}{\delta})}$ that
    \begin{align}
     \left\langle \bs_0, \mu\left(\frac{\bx^*_0}{\|\bx^*_0\|_2} - \bx^*_0\right) \right\rangle \le O\left(\sqrt{\frac{k\log\frac{Lr}{\delta}}{m}}\right) \mu \left\|\frac{\bx^*_0}{\|\bx^*_0\|_2} - \bx^*_0\right\|_2 \le O\left(\delta\sqrt{\frac{k\log\frac{Lr}{\delta}}{m}}\right),\label{eq:another_three_one}
    \end{align}
    where the last inequality uses \eqref{eq:norm_bound0}.
    Similar to that in the proof of Lemma~\ref{lem:bxStarbxHat}, we set $l = \lceil \log_2 n \rceil$. By~\eqref{eq:third_term}, the union bound over $G(M)$ (for $\bx^*_0$), and the assumption $\psi = \Theta(1)$, we obtain with probability $1- e^{-\Omega(k\log\frac{Lr}{\delta})}$ that
    \begin{equation}
     \left\langle \bs_0, \mu(\bx_l^* - \bx^*) \right\rangle \le O\left(\delta\sqrt{\frac{k\log\frac{Lr}{\delta}}{m}}\right).\label{eq:another_three_two}
    \end{equation}
    In addition, by~\eqref{eq:second_term} and a union bound over both $G(M)$ and over $G(M_{i-1}) \times G(M_i)$ for all $i \in [l]$, we obtain with probability $1- e^{-\Omega(k\log\frac{Lr}{\delta})}$ that
    \begin{equation}
     \sum_{i=1}^l \left\langle \bs_0, \mu(\bx_{i-1}^* - \bx^*_i) \right\rangle \le O\left(\delta\sqrt{\frac{k\log\frac{Lr}{\delta}}{m}}\right).\label{eq:another_three_three}
    \end{equation}
    Substituting \eqref{eq:total_three_terms_one}--\eqref{eq:another_three_three} into \eqref{eq:total_three_terms}, we obtain
    \begin{equation}
     \left\langle \bs_0, \hat{\bx} - \mu\bx^* \right\rangle \le  O\left(\delta+\sqrt{\frac{k\log\frac{Lr}{\delta}}{m}}\right) \|\hat{\bx} - \mu\bx^*\|_2 +O\left(\delta\sqrt{\frac{k\log\frac{Lr}{\delta}}{m}} \right). \label{eq:uniform_third_term}
    \end{equation}

    \item \underline{A uniform upper bound for $\big\langle \frac{1}{m}\bA^T\mu\bA\left(\frac{\bx^*_0}{\|\bx^*_0\|_2} - \bx^*\right), \hat{\bx} - \mu\bx^* \big\rangle$}: From Lemma~\ref{lem:boraSREC_gen}, we have that when $m = \Omega\left(k \log 
\frac{Lr}{\delta}\right)$, with probability $1-e^{-\Omega(m)}$,
    \begin{align}
    &\left\langle \frac{1}{m}\bA^T\mu\bA\left(\frac{\bx^*_0}{\|\bx^*_0\|_2} - \bx^*\right), \hat{\bx} - \mu\bx^* \right\rangle \nn \\
    & \quad\le \left\|\frac{1}{\sqrt{m}}\mu\bA \left(\frac{\bx^*_0}{\|\bx^*_0\|_2} - \bx^*\right)\right\|_2
     \left\|\frac{1}{\sqrt{m}}\bA (\hat{\bx} - \mu \bx^*)\right\|_2 \\
     &\quad \le O(\delta) O(\|\hat{\bx} - \mu\bx^*\|_2 + \delta) = O\left(\delta \|\hat{\bx} - \mu\bx^*\|_2 + \delta^2\right).\label{eq:uniform_fourth_term}
    \end{align}
\end{enumerate}
Having bounded the four terms, we now substitute \eqref{eq:uniform_first_term},~\eqref{eq:uniform_second_term},~\eqref{eq:uniform_third_term}, and~\eqref{eq:uniform_fourth_term} into \eqref{eq:four_terms}, and deduce that if $m \ge \funcM(\delta,\beta) + \Omega\big(k \log \frac{Lr}{\delta}\big)$, then with probability at least $1-e^{-\Omega(m)}-\funcP(\delta,\beta)$, it holds uniformly (in both $\mu\bx^*$ and $\hat{\bx}$) that
\begin{align}
 &\left\langle \frac{1}{m} \bA^T (\by - \mu\bA \bx^*), \hat{\bx} - \mu\bx^* \right\rangle  \nonumber \\
 & \quad\le O\left(\tau + \delta^\beta + \sqrt{\frac{k\log\frac{Lr}{\delta}}{m}}\right)\|\hat{\bx} - \mu\bx^*\|_2 + O\left(\delta \tau + \delta\sqrt{\frac{k\log\frac{Lr}{\delta}}{m}}  + \delta^{1+\beta}\right).
\end{align} 
Then, similarly to~\eqref{eq:non_unif_final}, we derive that if $m \ge \funcM(\delta,\beta) + \Omega\big(k \log \frac{Lr}{\delta}\big)$, then with probability at least $1-e^{-\Omega(m)}-\funcP(\delta,\beta)$, it holds uniformly that
\begin{align}
 \|\mu\bx^* - \hat{\bx}\|_2^2 \le O\left(\tau + \delta^\beta + \sqrt{\frac{k\log\frac{Lr}{\delta}}{m}}\right)\|\hat{\bx} - \mu\bx^*\|_2 + O\left(\delta \tau + \delta\sqrt{\frac{k\log\frac{Lr}{\delta}}{m}}  + \delta^{1+\beta}\right), \label{eq:combined}
\end{align}
where we used the fact that $\delta^\beta + \delta = O(\delta^{\beta})$, since $\beta \le 1$.

Considering the parameter $\epsilon$ in the theorem statement, we now set $\delta = \epsilon^{1/\beta}$ (i.e., $\epsilon = \delta^\beta$), meaning that the previous requirement $m = \Omega\left(\frac{k}{\epsilon^2}\log\frac{Lr}{\delta}\right)$ reduces to  $m = \Omega\left(\frac{k}{\epsilon^2}\log\frac{Lr}{\epsilon^{1/\beta}}\right) =\Omega\left(\frac{k}{\epsilon^2}\log\frac{Lr}{\epsilon}\right)$. In addition, $\sqrt{\frac{k\log\frac{Lr}{\delta}}{m}} = O(\epsilon)$. Since $\epsilon \le 1$ and $\beta \le 1$, we have
\begin{align}
 & O\left(\tau + \delta^\beta + \sqrt{\frac{k\log\frac{Lr}{\delta}}{m}}\right)\|\hat{\bx} - \mu\bx^*\|_2 + O\left(\delta \tau + \delta\sqrt{\frac{k\log\frac{Lr}{\delta}}{m}}  + \delta^{1+\beta}\right) \nonumber \\
 & \quad = O(\tau + \epsilon) \|\hat{\bx} - \mu\bx^*\|_2 + O\left(\epsilon^{1/\beta} \tau + \epsilon^{1+ 1/\beta}\right) \\
 & \quad = O(\tau + \epsilon) \|\hat{\bx} - \mu\bx^*\|_2 + O\left((\epsilon+ \tau)^2\right). \label{eq:final_two}
\end{align}
Substituting into \eqref{eq:combined} and considering two cases depending on which term in \eqref{eq:final_two} is larger, we obtain that if $m \ge \funcM(\epsilon^{1/\beta},\beta) + \Omega\big(\frac{k}{\epsilon^2} \log \frac{Lr}{\epsilon}\big)$, then with probability at least $1-e^{-\Omega(m)}-\funcP(\epsilon^{1/\beta},\beta)$, it holds uniformly that
\begin{equation}
 \|\mu\bx^* - \hat{\bx}\|_2 \le O(\tau+\epsilon).
\end{equation}

\end{document}